\documentclass[twoside,11pt]{article}
\usepackage[abbrvbib]{jmlr2e}
\usepackage{url}
\usepackage{amsmath}
\usepackage{amssymb}
\usepackage{thmtools,thm-restate}
\usepackage{siunitx}
\usepackage{setspace}
\usepackage{mathtools}
\usepackage{graphicx}
\usepackage{color}
\usepackage{enumerate}
\usepackage{algorithm}
\usepackage{algpseudocodex}
\usepackage{ifthen}
\usepackage{textcomp} 
\usepackage[labelformat=simple]{subcaption}

\usepackage{lastpage}
\jmlrheading{24}{2023}{1-\pageref{LastPage}}{11/21; Revised
6/23}{7/23}{21-1322}{Santtu Tikka, Jouni Helske, Juha Karvanen}
\ShortHeadings{Clustering and Structural Robustness}{Tikka, Helske and Karvanen}

\usepackage{tikz}
\usetikzlibrary{positioning, arrows.meta, shapes.geometric, shapes.misc}
\tikzset{%
  semithick,
  >={Stealth[width=2mm,length=2.75mm]},
  obs/.style = {name = #1, circle, draw, inner sep = 8pt, label = center:$#1$},
  col/.style 2 args = {name = #1, circle, color = #2, draw, inner sep = 8pt, label = center:\color{#2}$#1$},
  lat/.style = {name = #1, regular polygon, regular polygon sides = 4, draw,  inner sep = 7pt, label = center:$#1$}
}

\definecolor{violet}{rgb}{0.7,0,0.7}
\definecolor{gray}{rgb}{0.4,0.4,0.4}

\newcommand{\cond}{\,\vert\,}

\newcommand{\sV}{\mathcal{V}}

\newcommand{\sM}{\mathcal{M}}
\newcommand{\sG}{\mathcal{G}}

\newcommand{\sT}{\mathcal{T}}
\newcommand{\sF}{\mathcal{F}}
\newcommand{\sA}{\mathcal{A}}
\newcommand{\sB}{\mathcal{B}}
\newcommand{\sC}{\mathcal{C}}

\newcommand{\Pa}[1][]{%
  \ifthenelse{ \equal{#1}{} }
    {\textrm{Pa}}
    {\textrm{Pa}_{#1}}
}
\newcommand{\Ch}[1][]{%
  \ifthenelse{ \equal{#1}{} }
    {\textrm{Ch}}
    {\textrm{Ch}_{#1}}
}
\newcommand{\An}[1][]{%
  \ifthenelse{ \equal{#1}{} }
    {\textrm{An}}
    {\textrm{An}_{#1}}
}
\newcommand{\De}[1][]{%
  \ifthenelse{ \equal{#1}{} }
    {\textrm{De}}
    {\textrm{De}_{#1}}
}
\newcommand{\Ne}[1][]{%
  \ifthenelse{ \equal{#1}{} }
    {\textrm{Ne}}
    {\textrm{Ne}_{#1}}
}
\newcommand{\Co}[1][]{%
  \ifthenelse{ \equal{#1}{} }
    {\textrm{Co}}
    {\textrm{Co}_{#1}}
}
\newcommand{\rec}[1][]{%
  \ifthenelse{ \equal{#1}{} }
    {\textrm{Re}}
    {\textrm{Re}_{#1}}
}
\newcommand{\emi}[1][]{%
  \ifthenelse{ \equal{#1}{} }
    {\textrm{Em}}
    {\textrm{Em}_{#1}}
}

\newcommand{\compfinder}{\textsc{FindTrComp}}
\newcommand{\clustfinder}{\textsc{FindTrClust}}
\newcommand{\expander}{\textsc{ExpandClust}}
\AtBeginDocument{%
  \DeclareFontShape{OT1}{cmr}{m}{scit}{<->ssub*cmr/m/sc}{}%
}

\newcommand\independent{\protect\mathpalette{\protect\independenT}{\perp}} 
\def\independenT#1#2{\mathrel{\rlap{$#1#2$}\mkern2mu{#1#2}}}

\newcommand{\doo}{\textrm{do}}

\definecolor{colA}{RGB}{241,86,63} 
\definecolor{colB}{RGB}{0,82,174} 
\definecolor{colC}{RGB}{129,103,0}

\firstpageno{1}

\begin{document}

\title{Clustering and Structural Robustness in Causal Diagrams}

\author{\name Santtu Tikka \email santtu.tikka@jyu.fi \\
        \name Jouni Helske \email jouni.helske@jyu.fi \\
        \name Juha Karvanen \email juha.t.karvanen@jyu.fi \\
        \addr Department of Mathematics and Statistics \\
        P.O.Box 35 (MaD) FI-40014 University of Jyvaskyla, Finland}
\editor{Francis Bach}


\maketitle

\begin{abstract}%
Graphs are commonly used to represent and visualize causal relations. For a small number of variables, this approach provides a succinct and clear view of the scenario at hand. As the number of variables under study increases, the graphical approach may become impractical, and the clarity of the representation is lost. Clustering of variables is a natural way to reduce the size of the causal diagram, but it may erroneously change the essential properties of the causal relations if implemented arbitrarily. We define a specific type of cluster, called transit cluster, that is guaranteed to preserve the identifiability properties of causal effects under certain conditions. We provide a sound and complete algorithm for finding all transit clusters in a given graph and demonstrate how clustering can simplify the identification of causal effects. We also study the inverse problem, where one starts with a clustered graph and looks for extended graphs where the identifiability properties of causal effects remain unchanged. We show that this kind of structural robustness is closely related to transit clusters.
\end{abstract}


\section{Introduction} \label{sect:introduction}

Directed acyclic graphs (DAGs) and their extensions are commonly used to describe causal relations between variables in epidemiology and other fields \citep{pearl1995, Greenland1999, Tennant2021}. The power of graphs lies in their ability to visualize the assumed structure, and at the same time, to serve as well-defined inputs for algorithms such as those that solve the nonparametric identifiability of causal effects \citep{shpitser2006,lee2019,lee2020,tikka2021dosearch}. The graphical approach has been criticized by proponents of potential outcome framework \citep{Rubin1974} for its impracticality when a large number of variables is considered \citep{Imbens2020}. This criticism is partially justified: the visual clarity of a graph is easily lost when the number of vertices is more than a few, especially in the case of several crossing edges \citep{purchase1997}. Moreover, in some settings the identifiability of causal effects is an NP-hard problem \citep{tikka2020csi}, which makes it impractical to consider large graphs. A possible remedy for these difficulties is to cluster the variables to reduce the size of the graph. A question then arises whether the clustered graph and the original graph are equivalent with respect to the identifiability properties of causal effects.

The idea of clustering is natural and has been used in causal inference explicitly and implicitly. The back-door criterion \citep{pearl1993bayesian}, the front-door criterion \citep{pearl1995}, and ignorability assumptions in the potential outcome framework \citep{rosenbaum1983central} impose conditions upon a set (i.e., a cluster) of variables and the structure inside the set is not important. 
Explicitly, clusters have been constructed starting from structural equations \citep{skorstad1990clustered} or multivariate data \citep{entner2012estimating,parviainen2017learning,nisimov2021improving}. Outside causal inference,  many clustering methods for directed graphs have been proposed under varying premises  \citep{malliaros2013clustering}.


Clustering can be also viewed from a different starting point as a way to construct causal models where the causal relationships between clusters of variables are specified instead of the relationships between the variables themselves. For instance, a recent review \citep{Tennant2021} found that many DAGs in applied health research included so-called "super-nodes" \citep{Kornaropoulos2013} which represent multiple variables with the implicit assumption of strong connectivity of the corresponding vertices. This viewpoint emphasizes the uncertainty of structural assumptions and the fact that we may not possess sufficient knowledge about the domain under study to fully specify individual relationships between variables. The goal in this type of clustering is structural robustness; inferences made with the clustered graph can be safely applied in any graph that is compatible with the clustering, but not necessarily vice versa. This approach was considered in a formal setting by \citet{anand2021}. 


Clustering is different from latent projection \citep{verma1993graphical} that also can be used to simplify the structure of DAG in causal inference \citep{tikka2017}. Figure~\ref{fig:simplecluster} demonstrates that a cluster of vertices is not always equivalent to a latent projection in terms of identification. We consider the identifiability of a query $p(x_b \cond \doo(x_a))$ from observations $p(x_a,x_b,x_{c_1},x_{c_2})$. In this task, we may apply clustering $C = \{c_1,c_2\}$ and obtain an identifying formula similar to the original solution. However, if we use a latent projection to consider either $c_1$ or $c_2$ as unobserved in graph $\sG_1$, a bidirected edge between $a$ and $b$ appears, and the query is not identifiable anymore. In the graph $\sG_2$, the query is not identifiable if a latent projection is used to consider both $c_1$ and $c_2$ as unobserved, although projecting only either $c_1$ or $c_2$ retains the identifiability. On the other hand, arbitrary clustering of variables in a DAG does not necessarily retain the identifiability either.

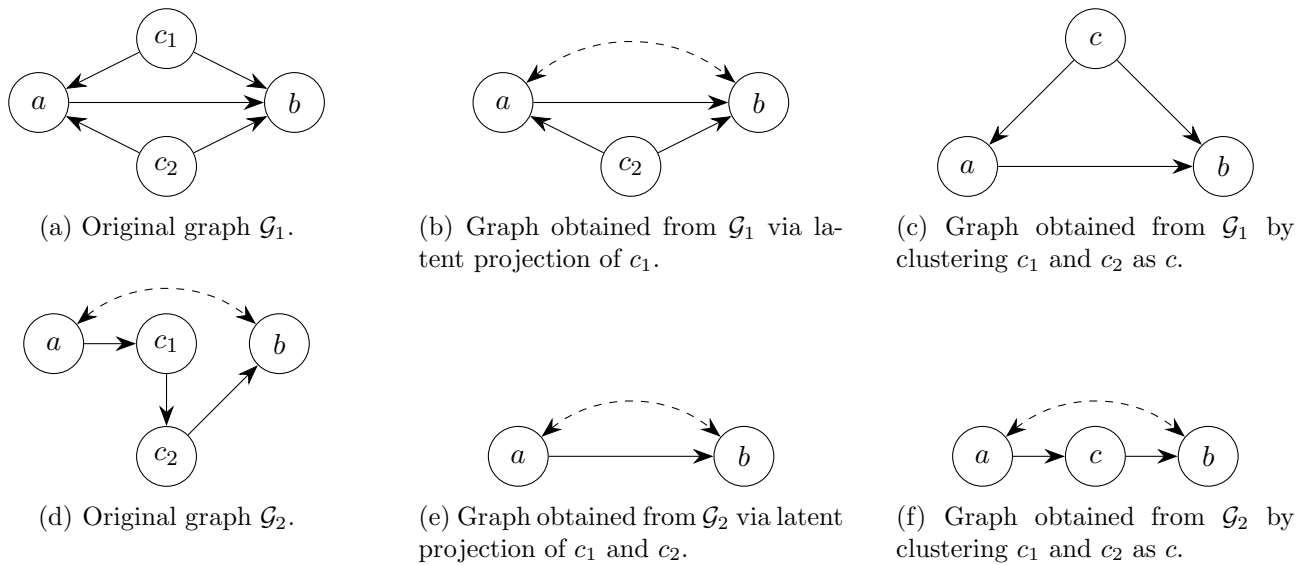
\begin{figure}[ht]
\begin{center}
\begin{subfigure}[t]{0.30\textwidth}
\centering
\begin{tikzpicture}[scale=1.7]
\node [obs = {a}] at (0,0) {};
\node [obs = {b}] at (2,0) {};
\node [obs = {c_1}] at (1,0.5) {};
\node [obs = {c_2}] at (1,-0.5) {};
\draw [->] (a) -- (b);
\draw [->] (c_1) -- (b);
\draw [->] (c_2) -- (b);
\draw [->] (c_1) -- (a);
\draw [->] (c_2) -- (a);
\end{tikzpicture}
\caption{Original graph $\sG_1$.} 
\end{subfigure}
\hfill
\begin{subfigure}[t]{0.32\textwidth}
\centering
 \begin{tikzpicture}[scale=1.7]
\node [obs = {a}] at (0,0) {};
\node [obs = {b}] at (2,0) {};
\node [obs = {c_2}] at (1,-0.5) {};
\draw [->] (a) -- (b);
\draw [->] (c_2) -- (b);
\draw [<->, dashed] (a) to [bend left=40]  (b);
\draw [->] (c_2) -- (a);
\end{tikzpicture}
\caption{Graph obtained from $\sG_1$ via latent projection of $c_1$.}
\end{subfigure}
\hfill
\begin{subfigure}[t]{0.30\textwidth}
 \centering
 \begin{tikzpicture}[scale=1.7]
\node [obs = {a}] at (0,0) {};
\node [obs = {b}] at (2,0) {};
\node [obs = {c}] at (1,1) {};
\draw [->] (a) -- (b);
\draw [->] (c) -- (a);
\draw [->] (c) -- (b);
\end{tikzpicture}
\caption{Graph obtained from $\sG_1$ by clustering $c_1$ and $c_2$ as $c$.}
\end{subfigure}

\begin{subfigure}[t]{0.3\textwidth}
 \centering
\begin{tikzpicture}[scale=1.5]
\node [obs = {a}] at (0,0) {};
\node [obs = {b}] at (2,0) {};
\node [obs = {c_1}] at (1,0) {};
\node [obs = {c_2}] at (1,-1) {};
\draw [->] (a) -- (c_1);
\draw [->] (c_1) -- (c_2);
\draw [->] (c_2) -- (b);
\draw [<->, dashed] (a) to [bend left=40]  (b);
\end{tikzpicture}
\caption{Original graph $\sG_2$.} 
\end{subfigure}
\hfill
\begin{subfigure}[t]{0.32\textwidth}
\centering
\begin{tikzpicture}[scale=1.5]
\node [obs = {a}] at (0,0) {};
\node [obs = {b}] at (2,0) {};
\draw [->] (a) -- (b);
\draw [<->, dashed] (a) to [bend left=40]  (b);
\end{tikzpicture}
\caption{Graph obtained from $\sG_2$ via latent projection of $c_1$ and $c_2$.}
\end{subfigure}
\hfill
\begin{subfigure}[t]{0.30\textwidth}
\centering
\begin{tikzpicture}[scale=1.5]
\node [obs = {a}] at (0,0) {};
\node [obs = {b}] at (2,0) {};
\node [obs = {c}] at (1,0) {};
\draw [->] (a) -- (c);
\draw [->] (c) -- (b);
\draw [<->, dashed] (a) to [bend left=40]  (b);
\end{tikzpicture}
\caption{Graph obtained from $\sG_2$ by clustering $c_1$ and $c_2$ as $c$.} 
\end{subfigure}
\caption{Two examples on clustering of vertices.} \label{fig:simplecluster}
\end{center}
\end{figure}

As the first contribution, we introduce a specific type of cluster, called transit cluster, and present conditions for the equivalence of causal effect identifiability between the original and the clustered graph. We consider clustering as an operation that transforms a DAG into a new DAG where the cluster is represented by a single vertex. Our approach toward clustering builds on the intuitive idea that information flows through a cluster and the detailed structure inside the cluster is often irrelevant. We assume that the DAG being clustered is fully specified.

As the second contribution, we provide a sound and complete algorithm for finding all transit clusters in a given graph and demonstrate how clustering can simplify the identification of causal effects.
While polynomial-time algorithms exist for many important causal identification problems, the resulting identifying functional can be complicated \citep{tikka2017simplifying}.
Clustering vertices in the graph can reduce the computational burden and lead to identifying functionals with a simpler structure.


As the third contribution, we study the inverse problem, where one starts with a clustered graph and looks for extended graphs where the identifiability properties of causal effects remain unchanged. This problem is related to the top-down causal modeling where one starts by creating the DAG with concepts, such as ``work history'', ``socio-economic background'' or ``genetic factors'', and only later divides these concepts into actual variables. Here a transit cluster represents this kind of concept. We present an iterative procedure that can be used to extend a single vertex to an arbitrary transit cluster. We show that transit clusters are structurally robust in the sense that under certain conditions the structure inside the cluster is irrelevant for identification. A schematic illustration of contributions of the paper is presented in Figure~\ref{fig:schematic}.

\begin{figure}[ht]
\begin{center}
\scalebox{0.72}{
  \begin{tikzpicture}[yscale = 1.65]
    \node at (-5, 0) (l1) {\large{\emph{Causal inference with large graphs}}};
    \node[draw, rounded corners, inner sep = 10] at (-5,-1) (l2) {Large unclustered graph};
    \node[draw, chamfered rectangle, chamfered rectangle xsep=2cm, align=center, fill = lightgray] at (-5,-2.5) (l3) {Clustering\\(Algorithms~\ref{alg:find_components} and \ref{alg:find_clusters})};
    \node at (-7.5, -4) {All transit clusters};
    \draw[draw=black, rounded corners] (-9.5,-7.15) rectangle ++(9.0, 3.40);
    \node[draw, rounded corners, inner sep = 8] at (-7.25,-4.75) (t1) {Transit cluster};
    \node[draw, rounded corners, inner sep = 8] at (-3,-4.75) (c1) {Clustered graph};
    \node[draw, rounded corners, inner sep = 8] at (-7.25,-5.5) (t2) {Transit cluster};
    \node[draw, rounded corners, inner sep = 8] at (-3,-5.5) (c2) {Clustered graph};
    \node at (-5.25,-6) {$\vdots$};
    \node[draw, rounded corners, inner sep = 8] at (-7.25,-6.5) (t3) {Transit cluster};
    \node[draw, rounded corners, inner sep = 8] at (-3,-6.5) (c3) {Clustered graph};
    \node at (5, 0) (r1) {\large{\emph{Top-down causal modeling}}};
    \node[draw, rounded corners, inner sep = 10] at (5,-1) (r2) {Clustered graph};
    \node[draw, chamfered rectangle, chamfered rectangle xsep=2cm, align=center, fill = lightgray] at (5,-2.5) (r3) {Peripheral extension\\(Definition~\ref{def:peripheral},\\Theorems~\ref{th:peripheral_induced} and \ref{th:peripheral_construction})};
    \node[draw, rounded corners, inner sep = 10] at (5,-4.1) (r4) {Extended graph};
    \node[draw, rounded corners, inner sep = 10] at (5,-5.35) (r5) {Unclustered graph};
    \draw [->] (l2) -- (l3);
    \draw [->] (l3) -- (-5, -3.75);
    \draw [->] (t1) -- (c1);
    \draw [->] (t2) -- (c2);
    \draw [->] (t3) -- (c3);
    \draw [->] (r2) -- (r3);
    \draw [->] (r3) -- (r4);
    \draw [->] (r4) -- (r5);
    \draw [->] (r4) -- (9, -4.1) -- (9,-2.5) -- (r3);
  \end{tikzpicture}
}
\caption{A schematic illustration of the contributions of the paper. In causal inference with large graphs (left), the application of the proposed clustering algorithm to a large graph results in a collection of all transit clusters, each of which corresponds to a clustered graph. In top-down causal modeling (right), we start with a clustered graph and iteratively apply peripheral extension to obtain an unclustered graph that shares the key properties with the clustered graph.} \label{fig:schematic}
\end{center}
\end{figure}
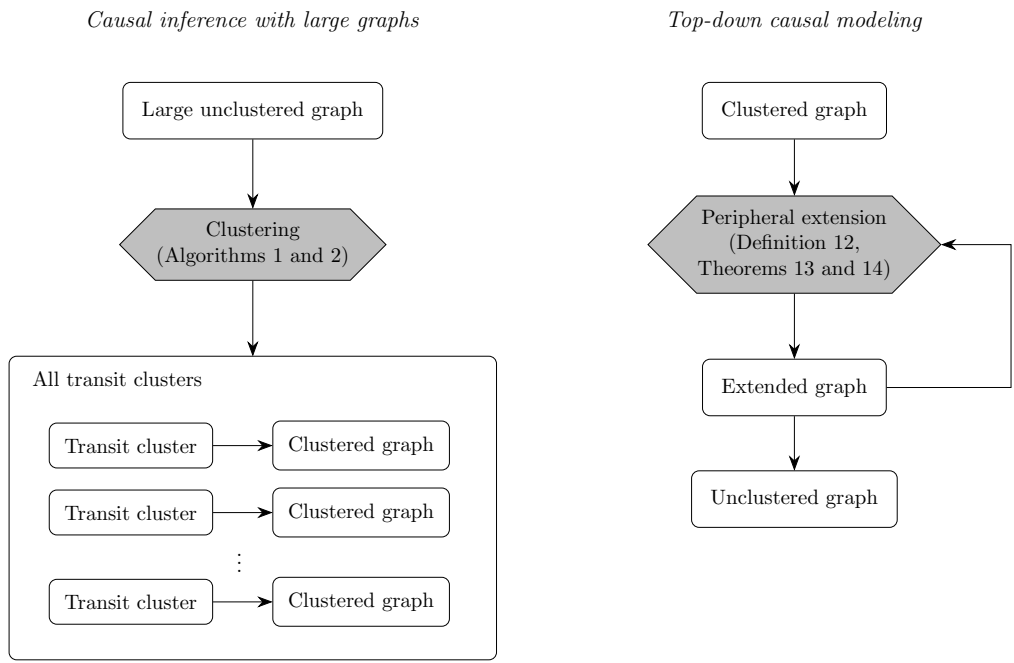


The rest of the paper is organized as follows. In Section~\ref{sec:cluster}, we define the transit cluster and prove its key properties. In Section~\ref{sect:algorithm}, we present an algorithm for finding all transit clusters of a DAG and prove that it is sound and complete. After considering clustering from a purely graphical point of view in Sections~\ref{sec:cluster} and \ref{sect:algorithm}, we then proceed to consider clustering in causal diagrams in Section~\ref{sec:transitcausal}, where we provide results on the identifiability of causal effects for specific transit clusters. In Section~\ref{sec:robustness}, we consider structural robustness and its connection to transit clusters. Illustrative examples on the clustering and structural robustness are given Section~\ref{sec:illustration}. Section~\ref{sec:discussion} concludes the paper. Code for the clustering algorithms and examples are available in a GitHub repository: \url{https://github.com/santikka/transit_cluster}.

\section{Clustering Vertices in DAGs} \label{sec:cluster}
We begin by introducing the notation used for directed graphs. A DAG $\sG = (V,E)$ is an ordered pair of two sets where $V$ is a set of indices (vertices), i.e., $V = \{1,\ldots,n\}$, and $E$ is a set of ordered pairs (directed edges) $E \subseteq \{(i,j) \mid i,j \in V\}$. Vertices and edges are denoted with small letters. 

In a DAG $\sG = (V,E)$, $\Pa[\sG](A)$, $\Ch[\sG](A)$, $\An[\sG](A)$ and $\De[\sG](A)$ denote the parents, children, ancestors and descendants of vertex set $A \subseteq V$ including $A$, respectively. The neighbors of a vertex set $A$ including $A$ is denoted by $\Ne[\sG](A) \equiv \Ch[\sG](A) \cup \Pa[\sG](A)$. Vertices connected to $A$ including $A$ is denoted by $\Co[\sG](A)$. The corresponding sets of the previous that exclude $A$ are denoted by $\Pa[\sG]^*(A)$, $\Ch[\sG]^*(A)$, $\An[\sG]^*(A)$, $\De[\sG]^*(A)$, $\Ne[\sG]^*(A)$, and $\Co[\sG]^*(A)$. If there is only one relevant graph $\sG$ in a given context, we will sometimes omit the subscript from these sets for clarity, and simply write $\Pa(A)$ or $\Pa^*(A)$, for example. 

We use the notation $\sG[W]$ to denote an induced subgraph $(W, F)$ of $\sG = (V,E)$, where $W \subseteq V$ and $F$ contains those edges of $E$ with both endpoints in $W$. Similarly, $\sG[\overline A, \underline B]$ denotes an induced edge subgraph obtained from $\sG$ by removing incoming edges to $A \subseteq V$ and outgoing edges of $B \subseteq V$. The collection of vertex sets that induce the components of $\sG$ is denoted by $\sC(\sG)$.

By a cluster we mean a subset of vertices of a DAG. Note that there are different definitions of ``clustering'' and ``cluster graph'' in other contexts. The motivation for the name ``cluster'' becomes evident when we consider a graph that represents the cluster as a single vertex:

\begin{definition}[Clustering] \label{def:clustering} 
Clustering of a set of vertices $T \subset V$ in a DAG $\sG=(V,E)$ induces a graph $\sG^\prime=(V^\prime,E^\prime)$ obtained from $\sG$ by removing vertices $T$ and adding a new vertex $t$ that has parents  $\Pa^*_{\sG}(T)$ and children $\Ch^*_{\sG}(T)$. In addition, sets $W \subset V$ and $W^\prime \subset V^\prime$ are clustering equivalent if $W \setminus T = W^\prime \setminus \{t\}$ and $T \subset W$ if and only if $t \in W^\prime$. 
\end{definition}

Definition~\ref{def:clustering} captures the intuitive idea of clustering where the incoming and outgoing edges of the cluster are the same as the incoming and outgoing edges of its representative in the induced graph. However, without any constraints on the set $T$ being clustered, this definition is too general in the sense that the properties of the induced graph may be drastically different from the original graph. For example, the induced graph is not necessarily a DAG or it may contain paths that were not present in the original graph.

Next, we will present conditions for the clustered vertices $T$ that guarantee the usefulness of the clustering. Our approach is based on the intuitive notion that the effects flow through the cluster and the edges between the clustered vertices do not matter. Only those edges that connect to vertices outside the cluster are relevant. For this purpose, we define two special sets of vertices.

\begin{definition}[Receiver] For a set of vertices $T$ in a DAG $\sG = (V,E)$, the set of \emph{receivers} is the set
\[
  \rec[\sG](T) \equiv \{v \in T \mid \Pa[\sG](v) \cap (V \setminus T) \neq \emptyset\}.
\]
\end{definition}
The set of receivers for a set of vertices $T \subset V$ are those members of $T$ that have parents outside of $T$ in $\sG$. To complement the receivers, we also define the following.
\begin{definition}[Emitter] For a set of vertices $T$ in a DAG $\sG = (V, E)$, the set of $\emph{emitters}$ is the set
\[
  \emi[\sG](T) \equiv \{v \in T \mid \Ch[\sG](v) \cap (V \setminus T) \neq \emptyset\}.
\]
\end{definition}

We will use shortcut notation $\sG[T^=]$ to denote a subgraph induced by $T$ such that the incoming edges to receivers of $T$ and outgoing edges from emitters of $T$ are removed. We are now ready to define a cluster that preserves the fundamental structure of the graph.

\begin{definition}[Transit cluster] \label{def:transitcluster}
A non-empty set $T \subset V$ is a \emph{transit cluster} in a connected DAG $\sG = (V,E)$ if the following conditions hold
\begin{enumerate} 
  \item \label{itm:parents} $\Pa(r_i) \setminus T = \Pa(r_j) \setminus T$ for all pairs $r_i,r_j \in \rec(T)$,
  \item \label{itm:children} $\Ch(e_i) \setminus T = \Ch(e_j) \setminus T$ for all pairs $e_i,e_j \in \emi(T)$, 
  \item \label{itm:no_ronsy_allowed} For all vertices $t_i \in T$, there exists a receiver $r$ or an emitter $e$ such that $t_i$ and $r$ or $t_i$ and $e$ are connected via an undirected path in $\sG[T^=]$.
  \item \label{itm:forallr} If $\emi(T) \neq \emptyset$ then for all $r \in \rec(T)$ there exists $e \in \emi(T)$ such that $e \in \De(r)$,
  \item \label{itm:foralle} If $\rec(T) \neq \emptyset$ then for all $e \in \emi(T)$ there exists $r \in \rec(T)$ such that $r \in \An(e)$.
\end{enumerate}
\end{definition}

In other words, a transit cluster is a set of vertices such that any member of its receivers has the same parents outside of the set, and any member of its emitters has the same children outside of the set. Additionally, we disallow those vertices from belonging to the cluster that are disconnected from the receivers or emitters when incoming edges of receivers and outgoing edges of emitters have been removed. Finally, for any receiver, there is always a directed path connecting that receiver to an emitter, and conversely, for any emitter, there is always a directed path connecting a receiver to that emitter. Together, these features endow the grouped set of vertices with several desirable properties. The set of all transit clusters of $\sG$ is denoted by $\Upsilon_\sG$.

The purpose behind the first two conditions in Definition~\ref{def:transitcluster} is to ensure that no new paths from the parents of the receivers to the children of emitters are created by performing the clustering. The third condition ensures that a transit cluster is characterized by its receivers and emitters, as we will show later. The last two conditions enforce the idea of information flow through the cluster. Examples of transit clusters are presented in Figure~\ref{fig:transitclusterexamples}.

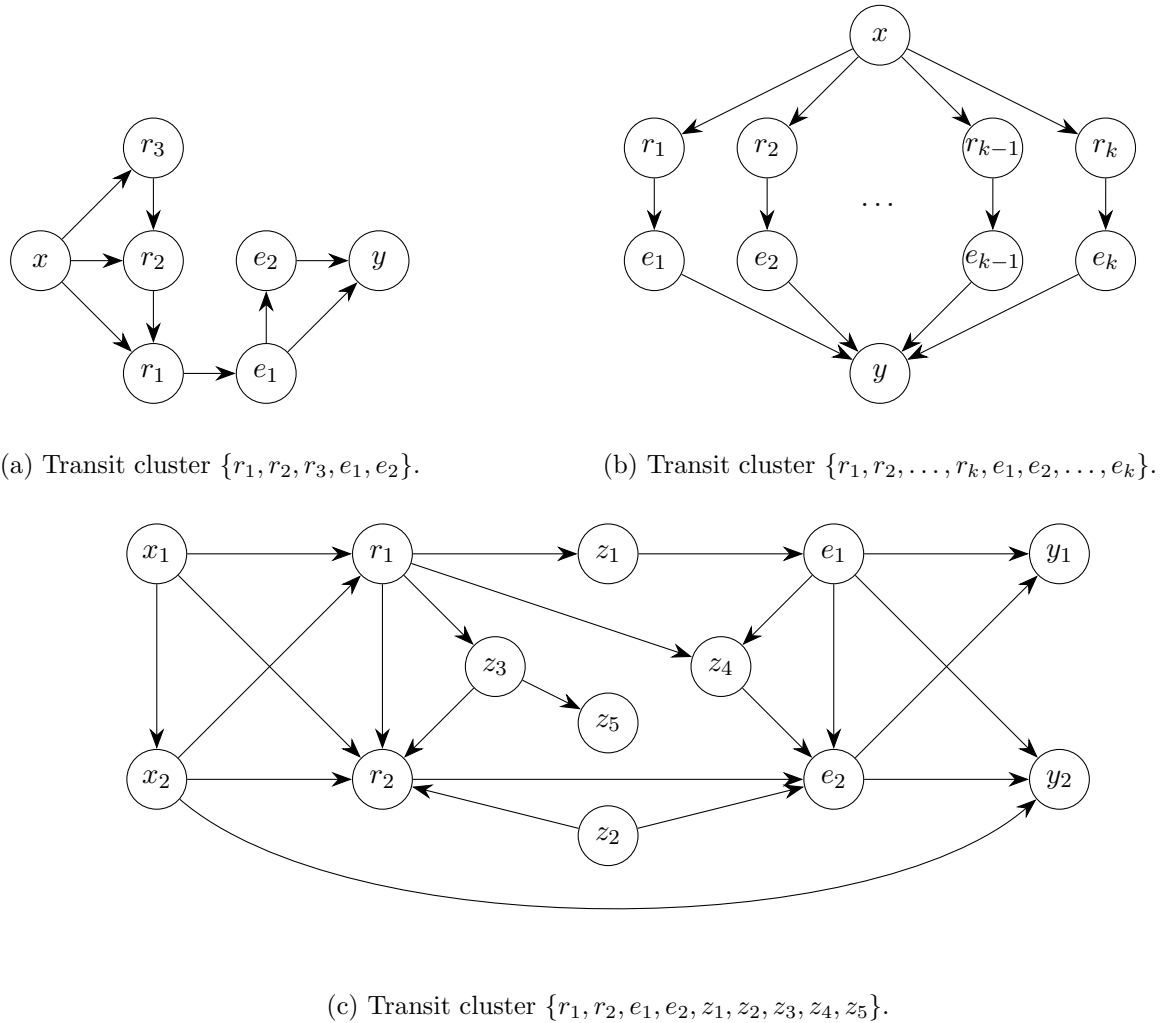
\begin{figure}[!ht]
\begin{subfigure}[t]{0.40\textwidth}
\begin{center}
  \begin{tikzpicture}[scale=1.5]
  \node [obs = {x}] at (0,2) {};
  \node [obs = {r_1}] at (1,1) {};
  \node [obs = {r_2}] at (1,2) {};
  \node [obs = {r_3}] at (1,3) {};
  \node [obs = {e_1}] at (2,1) {};
  \node [obs = {e_2}] at (2,2) {};
  \node [obs = {y}] at (3,2) {};
  \draw[->] (x) -- (r_1);
  \draw[->] (x) -- (r_2);
  \draw[->] (x) -- (r_3);
  \draw[->] (r_1) -- (e_1);
  \draw[->] (r_3) -- (r_2);
  \draw[->] (r_2) -- (r_1);
  \draw[->] (e_1) -- (e_2);
  \draw[->] (e_1) -- (y);
  \draw[->] (e_2) -- (y);
  \end{tikzpicture}
\end{center}
\caption{Transit cluster $\{r_1,r_2,r_3,e_1,e_2\}$.} \label{fig:minimal_representative}
\end{subfigure}
\hfill
\begin{subfigure}[t]{0.59\textwidth}
\begin{center}
  \begin{tikzpicture}[scale=1.5]
  \node [obs = {x}] at (3,4) {};
  \node [obs = {r_1}] at (1,3) {};
  \node [obs = {e_1}] at (1,2) {};
  \node [obs = {r_2}] at (2,3) {};
  \node [obs = {e_2}] at (2,2) {};
   \node [obs = {r_{k-1}}] at (4,3) {};
  \node [obs = {e_{k-1}}] at (4,2) {};
  \node [obs = {r_k}] at (5,3) {};
  \node [obs = {e_k}] at (5,2) {};
  \node [obs = {y}] at (3,1) {};  
  \node at ($(x)!.5!(y)$) {\ldots};
  \draw[->] (x) -- (r_1);
  \draw[->] (x) -- (r_2);
  \draw[->] (x) -- (r_{k-1});
  \draw[->] (x) -- (r_k);
  \draw[->] (r_1) -- (e_1);
  \draw[->] (r_2) -- (e_2);
  \draw[->] (r_{k-1}) -- (e_{k-1});
  \draw[->] (r_k) -- (e_k);
  \draw[->] (e_1) -- (y);
  \draw[->] (e_2) -- (y);
  \draw[->] (e_{k-1}) -- (y);
  \draw[->] (e_k) -- (y);
  \end{tikzpicture}
\end{center}
\caption{Transit cluster $\{r_1,r_2,\ldots,r_k,e_1,e_2,\ldots,e_k \}$.} \label{fig:numberoftransitclusters}
\end{subfigure}

\vspace{0.5cm}
\begin{subfigure}[t]{\textwidth}
\begin{center}
 \begin{tikzpicture}[xscale=3.0, yscale=1.5]
\node [obs = {x_1}] at (0,1) {};
\node [obs = {r_1}] at (1,1) {};
\node [obs = {z_1}] at (2,1) {};
\node [obs = {e_1}] at (3,1) {};
\node [obs = {y_1}] at (4,1) {};
\node [obs = {x_2}] at (0,-1) {};
\node [obs = {z_3}] at (1.5,0) {};
\node [obs = {z_4}] at (2.5,0) {};
\node [obs = {z_5}] at (2,-0.5) {};
\node [obs = {r_2}] at (1,-1) {};
\node [obs = {z_2}] at (2,-1.5) {};
\node [obs = {e_2}] at (3,-1) {};
\node [obs = {y_2}] at (4,-1) {};

\draw [->] (x_1) -- (r_1);
\draw [->] (x_1) -- (r_2);
\draw [->] (x_1) -- (x_2);
\draw [->] (r_1) -- (z_1);
\draw [->] (r_1) -- (z_3);
\draw [->] (z_1) -- (e_1);
\draw [->] (z_3) -- (r_2);
\draw [->] (z_3) -- (z_5);
\draw [->] (e_1) -- (y_1);
\draw [->] (e_1) -- (y_2);
\draw [->] (e_1) -- (z_4);
\draw [->] (z_4) -- (e_2);
\draw [->] (x_2) -- (r_1);
\draw [->] (x_2) -- (r_2);
\draw [->] (r_2) -- (e_2);
\draw [->] (e_2) -- (y_1);
\draw [->] (e_2) -- (y_2);
\draw [->] (r_1) -- (r_2);
\draw [->] (r_1) -- (z_4);
\draw [->] (e_1) -- (e_2);
\draw [->] (z_2) -- (e_2);
\draw [->] (z_2) -- (r_2);
\draw [->] (x_2) to [bend right=60]  (y_2) ;
\end{tikzpicture}
\end{center}
\caption{Transit cluster $\{r_1,r_2,e_1,e_2,z_1,z_2,z_3,z_4,z_5\}$.} \label{fig:exampletransitclusters}
\end{subfigure}
\caption{Examples of transit clusters. In addition to presented transit clusters, there are other transit clusters, for instance transit cluster $\{r_1,e_1\}$ in panel (a).}
\label{fig:transitclusterexamples}
\end{figure}

Definitions~\ref{def:clustering}--\ref{def:transitcluster} allow us to characterize the graph induced by a transit cluster as follows: 
\begin{corollary}
The induced graph $\sG^\prime$ of transit cluster $T$ is constructed from $\sG$ by replacing $T$ with a single vertex $t$ such that $\Pa[\sG^\prime]^*(t) = \Pa[\sG]^*(\rec[\sG](T)) \setminus T$ and $\Ch[\sG^\prime]^*(t) = \Ch[\sG]^*(\emi[\sG](T)) \setminus T$.
\end{corollary}

We consider some desirable basic properties of transit clusters. 
We delegate the proofs of all results to Appendix~\ref{app:basic_proofs}. First, we must ensure that the graph induced by a transit cluster does not contain cycles.
\begin{restatable}{lem}{isdag}
\label{lem:isdag} 
Graph $\sG^\prime$ induced by a transit cluster $T$ in a DAG $\sG$ is a DAG.
\end{restatable}
Next, we note that transit clusters are uniquely defined by their receivers and emitters.
\begin{restatable}{lem}{reem} 
\label{lem:re_em}
Let $T$ and $S$ be transit clusters in a DAG $\sG = (V,E)$. If $\rec(T) = \rec(S)$ and $\emi(T) = \emi(S)$, then $T = S$.
\end{restatable}
Intuitively, if clustering is carried out for multiple vertex sets in sequence, the order in which the clustering is carried out should not matter in terms of the graph obtained after the last set has been clustered. This notion is captured by the next two theorems. The first one states that a transit cluster remains a transit cluster even if a disjoint transit cluster is clustered.

\begin{restatable}[Invariance of transit clusters]{theorem}{stilltransitcluster} 
\label{th:stilltransitcluster}
Let $T$ be a transit cluster in $\sG=(V,E)$ and let $\sG^\prime$ be the induced graph where $T$ is replaced by a single vertex $t$. The set $S \subset V \setminus T$ is a transit cluster in $\sG$ if and only if it is a transit cluster in $\sG^\prime$.
\end{restatable}
A complementary result to the previous theorem guarantees that a transit cluster will still be a transit cluster even if its subset is clustered.

\begin{restatable}[Modularity of transit clusters]{theorem}{modularity} 
\label{th:modularity}
Let $T$ be a transit cluster in graph $\sG=(V,E)$ and $\sG^\prime$ the induced graph where $T$ is replaced by a single vertex $t$. Let $S \subset V \setminus T$. The set $\{t\} \cup S$ is a transit cluster in $\sG^\prime$ if and only if $T \cup S$ is a transit cluster in $\sG$.
\end{restatable}
Theorem~\ref{th:modularity} helps us to characterize the conditions for valid unions of transit clusters. The following corollary plays a key role in the algorithmic approach to finding transit clusters in Section~\ref{sect:algorithm}.

\begin{restatable}[Union of transit clusters]{cor}{clusterunion}
\label{th:union}
Let disjoint sets $S$ and $T$ be transit clusters in $\sG$. $S \cup T$ is a transit cluster in $\sG$ if
$\Pa^*(\rec(S)) = \Pa^*(\rec(T))$ and $\Ch^*(\emi(S)) = \Ch^*(\emi(T))$.
\end{restatable}
While other types of unions of transit clusters can result in valid transit clusters, it turns out that the particular union specified by Corollary~\ref{th:union} is the only one we actually need.

In a practical setting, there may be vertices that we cannot or do not want to include in the same cluster with other vertices. Thus the set of possible clusters may be restricted.

\begin{definition}[Restricted transit cluster] \label{def:restricted}
Let $\sG = (V,E)$ be a DAG and $R \subseteq V$. A restricted transit cluster $T\subset V$ with respect to $R$ in $\sG$ is a transit cluster in $\sG$ such that $T \subseteq R$.
\end{definition}
We denote the set of all restricted transit clusters with respect to $R$ by $\Upsilon_{\sG|R} \equiv \{ T \mid T \in \Upsilon_\sG, T \subseteq R \}$. The next definition specifies the operations that can be applied to a DAG so that a transit cluster remains a transit cluster. 

\begin{definition}[Peripheral extension] \label{def:peripheral}
Let $T = \{t_1,\ldots,t_k\}$ be a transit cluster in a DAG $\sG$. Let $\sG^+$ be a DAG obtained from $\sG$ by one of the following operations:
\begin{enumerate}
\item \label{peri:addedge} Add an edge $t_i \rightarrow t_j$ where $t_i, t_j \in T$ and the edge does not create a cycle,
\item \label{peri:addmediator} Replace the edge $t_i \rightarrow t_j$ 
by the path $t_i \rightarrow t_{k+1} \rightarrow t_j$ where $t_{k+1}$ is a new vertex,
\item \label{peri:dividenode} Divide vertex $t_i$ as follows: add a new vertex $t_{k+1}$ such that $\Ch^*_{\sG^+}(t_{k+1}) = \Ch^*_{\sG}(t_i)$, remove all outgoing edges of $t_i$ and add edge  $t_{i} \rightarrow t_{k+1}$.
\item \label{peri:addparent} Add a new vertex $t_{k+1}$ and the edge $t_{k+1} \rightarrow t_i$ and  where $t_i \in T \setminus \rec(T)$,
\item  \label{peri:addchild}  Add a new vertex $t_{k+1}$ and the edge $t_i \rightarrow t_{k+1}$ where $t_i \in T \setminus \emi(T)$.
\end{enumerate}
Adding a receiver or an emitter when $\rec(T) \neq \emptyset$ and $\emi(T) \neq \emptyset$\,\textnormal{:}
\begin{enumerate}
\setcounter{enumi}{5}
\item \label{peri:addrec} Add a new vertex $t_{k+1}$  (a new receiver) such that  $\Pa^*_{\sG^+}(t_{k+1}) = \Pa^*_{\sG}(\rec(T))$, and add an edge $t_{k+1} \rightarrow t_j$ where $t_j \in T$ is on a path from a receiver to an emitter  in $\sG$.
\item \label{peri:addemi} Add a new vertex $t_{k+1}$ (a new emitter) such that  $\Ch^*_{\sG^+}(t_{k+1}) = \Ch^*_{\sG}(\emi(T))$, and add an edge $t_j  \rightarrow t_{k+1}$ where $t_j \in T$ is on a path from a receiver to an emitter in $\sG$.
\item \label{peri:addrecemi} Add two new vertices $t_{k+1}$ (a new receiver) and $t_{k+2}$ (a new emitter) such that  $\Pa^*_{\sG^+}(t_{k+1}) = \Pa^*_{\sG}(\rec(T))$ and  $\Ch^*_{\sG^+}(t_{k+2}) = \Ch^*_{\sG}(\emi(T))$, and add the edge $t_{k+1} \rightarrow t_{k+2}$.
\end{enumerate}
Adding a receiver  when $\rec(T) \neq \emptyset$ and $\emi(T) = \emptyset$\,\textnormal{:}
\begin{enumerate}
\setcounter{enumi}{8}
\item \label{peri:addrec_noemi} Add a new vertex $t_{k+1}$  (a new receiver) such that  $\Pa^*_{\sG^+}(t_{k+1}) = \Pa^*_{\sG}(\rec(T))$.
\end{enumerate}
Adding an emitter when $\rec(T) = \emptyset$ and $\emi(T) \neq \emptyset$\,\textnormal{:}
\begin{enumerate}
\setcounter{enumi}{9}
\item \label{peri:addemi_norec} Add a new vertex $t_{k+1}$ (a new emitter) such that  $\Ch^*_{\sG^+}(t_{k+1}) = \Ch^*_{\sG}(\emi(T))$.
\end{enumerate}
Now define $T^+ = T$ if operation~\ref{peri:addedge} was applied,  $T^+ = T \cup \{ t_{k+1} \}$ if operation~\ref{peri:addmediator}, \ref{peri:dividenode}, \ref{peri:addparent}, \ref{peri:addchild}, \ref{peri:addrec}, \ref{peri:addemi}, \ref{peri:addrec_noemi} or \ref{peri:addemi_norec} was applied, and $T^+ = T \cup \{ t_{k+1}, t_{k+2} \}$ if operation~\ref{peri:addrecemi} was applied. In all cases, $T^+$ is a \emph{peripheral extension} of $T$, and $\sG^+$ is the corresponding peripheral extension graph.
\end{definition}
Operations 1--5 modify the internal structure of the transit cluster. New vertices and edges can be added but new parents for receivers or new children for emitters cannot be added with these operations. Operations 6--10 add new receivers and emitters. Here the allowed operations differ for transit clusters that have only receivers, only emitters and both receivers and emitters. Special conditions are needed to make sure that $T^+$ fulfills the conditions of Definition~\ref{def:transitcluster}. The following theorem shows that a peripheral extension always results in a transit cluster. 

\begin{restatable}{theorem}{peripheralinduced}
\label{th:peripheral_induced} Let $T^{+}$ be a peripheral extension of a transit cluster $T$ in a DAG $\sG = (V,E)$ and let $\sG^\prime$ be the induced graph where $T$ is replaced by a single vertex. Then $T^{+}$ is a transit cluster in the corresponding peripheral extension graph $\sG^+$ and $\sG^\prime$ is the induced graph of $T^{+}$.
\end{restatable}
A complementary result shows that any transit cluster can be constructed iteratively with operations of Definition~\ref{def:peripheral}.

\begin{restatable}{theorem}{peripheralconstruction}
\label{th:peripheral_construction}
Let $\sG^\prime$ be the induced graph of a transit cluster constructed from $\sG$ by replacing set $T$ by a single vertex $t$. Then $\sG$ and $T$ can be constructed by iteratively applying operations of Definition~\ref{def:peripheral} to transit cluster $\{t\}$ in $\sG^\prime$.
\end{restatable}
So far, our focus has been on transit clusters in general. It turns out that transit clusters can always be constructed from smaller ``building blocks'' which we call transit components. Furthermore, it is much easier to find transit components in a given DAG than transit clusters.

\begin{definition}[Transit component] \label{def:transit_component}
  A transit cluster $T$ in a DAG $\sG$ is a transit component if $T$ is connected in $\sG[T]$.
\end{definition}
We extend the notion of restriction to transit components: the set of all restricted transit components of a DAG $\sG$ with respect to the set $R$ is denoted by $\sT_{\sG|R}$, and when $R = V$, i.e., when the aforementioned set corresponds to components without restriction, we simply write $\sT_\sG$. The intuitive idea behind the purpose of transit components is encapsulated in the following theorem.

\begin{restatable}[Transit cluster decomposition]{theorem}{clusterdecomposition}
\label{thm:cluster_is_component_union}
Let $T$ be a transit cluster in a DAG $\sG = (V,E)$. If $T$ is not a transit component, then there exists a transit cluster $S$ and a transit component $R$ such that $T = S \cup R$ and $S \cap R = \emptyset$.
\end{restatable}
In simpler terms, any transit cluster can always be constructed iteratively from disjoint transit components. The transit cluster decomposition plays a key role in finding transit clusters.

\section{Clustering Algorithms} \label{sect:algorithm}

To find valid vertex clusters, we can always apply a naive approach and enumerate every vertex subset of a DAG $\sG$ and determine whether the conditions of Definition~\ref{def:transitcluster} hold. However, this approach quickly becomes infeasible with larger graphs. In light of Theorem~\ref{thm:cluster_is_component_union}, we can instead start by constructing the set of all transit components, and then obtain the set of all transit clusters by applying Corollary~\ref{th:union} to the components. We begin by presenting a sound and complete algorithm for finding transit components that exploit the structure of the graph by enumerating a set of candidate receiver and emitter sets. 

Algorithm~\ref{alg:find_components} (\compfinder{}) starts by constructing the candidate sets for potential receivers and emitters, $\sV_{\text{Ch}}$ and $\sV_{\text{Pa}}$, respectively on lines~\ref{line:construct_receiver_candidates} and \ref{line:construct_emitter_candidates}. Lemma~\ref{lem:re_em} shows that we only need to consider the receivers and emitters to uniquely specify a transit component. Next, we iterate over all pairs of the candidates on lines~\ref{line:loop_receiver_candidates} and \ref{line:loop_emitter_candidates}. Lines~\ref{line:restrict_an}--\ref{line:has_ch_or_pa} restrict the candidates into mutually ancestral sets $Z$ and $W$, and further exclude those candidates that cannot satisfy the properties of a transit cluster. If at least one of the obtained candidate sets $Z$ and $W$ is nonempty on line~\ref{line:re_or_em_nonempty}, we move on to construct a candidate transit component $A$ on line~\ref{line:defineA}. If $A$ obeys the restriction defined by $R$ on line~\ref{line:A_is_restricted}, we move on to the iteration over the components of $A$ in the induced subgraph $\sG[A]$ on line~\ref{line:component_iteration}. Lines~\ref{line:component_re} and \ref{line:component_em} define those members of the current candidates $Z$ and $W$ that belong to the current component $A_k$ as $Z_k$ and $W_k$ respectively. Finally, we determine whether the members of $Z_k$ have the same parents, and whether members of $W_k$ have the same children during lines~\ref{line:common_pa}--\ref{line:re_em_verify}. If this is the case, a new transit component of $\sG$ has been found, and it is added to the set $\sA$ of components found so far. Finally, this set is returned on line~\ref{line:return_components} after the outermost iterations have been completed. 

\begin{algorithm}[!ht]
  \begin{algorithmic}[1]
  \Function{FindTrComp}{$\sG, R$}
    \State $\sA \gets \emptyset$
    \State $\sV_{\text{Ch}} \gets \{ \Ch^*(v) \cap R \mid v \in V \}$ \label{line:construct_receiver_candidates}
    \State $\sV_{\text{Pa}} \gets \{ \Pa^*(v) \cap R \mid v \in V \}$ \label{line:construct_emitter_candidates}
    \ForAll{$V_i \in \sV_{\text{Ch}}$} \label{line:loop_receiver_candidates}
      \ForAll{$V_j \in \sV_{\text{Pa}}$} \label{line:loop_emitter_candidates}
        \State \textbf{if} $V_j \neq \emptyset$ \textbf{then} $Z \gets V_i \cap \An[\sG](V_j)$ \textbf{else} $Z \gets V_i$ \label{line:restrict_an}
        \State \textbf{if} $V_i \neq \emptyset$ \textbf{then} $W \gets V_j \cap \De[\sG](V_i)$ \textbf{else} $W \gets V_j$ \label{line:restrict_de}
        \If{$(\exists z_k\in Z \text{ s.t. } \Pa[\sG]^*(z_k) = \emptyset) \vee (\exists w_k\in W \text{ s.t. } \Ch[\sG]^*(w_k) = \emptyset)$} \label{line:has_ch_or_pa}
          \State \bf{continue}
        \EndIf
        \If{$Z \neq \emptyset \vee W \neq \emptyset$} \label{line:re_or_em_nonempty}
          \State $A \gets \Co[{\sG[\overline {Z},\underline {W}]}](Z \cup W)$ \label{line:defineA}
          \If{$A \subseteq R$} \label{line:A_is_restricted}
            \ForAll{$A_k \in \sC(\sG[A])$} \label{line:component_iteration}
              \State $Z_k \gets Z \cap A_k$ \label{line:component_re}
              \State $W_k \gets W \cap A_k$ \label{line:component_em}
              \State $Z_{\text{Pa}} \gets \bigcap_{z \in Z_k} \Pa[\sG]^*(z) \setminus A_k$ \label{line:common_pa}
              \State $W_{\text{Ch}} \gets \bigcap_{w \in W_k} \Ch[\sG]^*(w) \setminus A_k$ \label{line:common_ch}
              \If{$Z_{\text{Pa}} = \Pa^*(Z_k) \setminus A_k \wedge W_{\text{Ch}} = \Ch^*(W_k) \setminus A_k$} \label{line:re_em_verify}
                \State $\sA \gets \sA\, \cup \{A_k\}$ \label{line:add_component}
              \EndIf
            \EndFor
          \EndIf
        \EndIf
      \EndFor
    \EndFor
    \State \Return $\mathcal{A}$ \label{line:return_components}
  \EndFunction
  \end{algorithmic}
  \caption{An algorithm for finding all (restricted) transit components of a DAG. The inputs are a DAG $\sG = (V,E)$ and a restriction set $R$ whose members are the only vertices in $V$ that are allowed to belong to any transit component. Returns the set of all restricted transit components in $\mathcal{G}$ with respect to $R$.}
  \label{alg:find_components}
\end{algorithm}

We proceed to show that \compfinder{} always terminates.
\begin{restatable}{lem}{compterminates}
\label{lem:comp_terminates}
\compfinder{} always terminates for valid inputs $\sG$ and $R$.
\end{restatable}

Lemma~\ref{lem:comp_terminates} allows us to consider the output of \compfinder{}. To show soundness of the algorithm, me must show that the output set only contains restricted transit components.
\begin{restatable}[Soundness of \compfinder{}]{theorem}{compfindersound}
Let $\sG = (V,E)$ be a DAG, $R \subseteq V$ and $\sA = \compfinder(\sG, R)$, then $\sA \subseteq \sT_{\sG|R}.$
\end{restatable}
Conversely for completeness of \compfinder{}, we must show that any restricted transit component of a DAG will be found by the algorithm.
\begin{restatable}[Completeness of \compfinder{}]{theorem}{compfindercomplete}
Let $\sG = (V,E)$ be a DAG, $R \subseteq V$, and $\sA = \compfinder(\sG, R)$ then $\sT_{\sG|R} \subseteq \sA$.
\end{restatable}

\compfinder{} operates in polynomial time with respect to the size of the graph.
\begin{restatable}{theorem}{comppoly}
\label{th:comp_poly}
\compfinder{} outputs all restricted transit components of a DAG $G = (V,E)$ with respect to $R \subseteq V$ in $O(|V|^4 + |V|^3|E|)$ time.
\end{restatable}

Theorem~\ref{th:comp_poly} also gives an upper bound for the number of distinct transit components of a DAG. To obtain a crude approximation, we note that there are $|V|^2$ total candidate pairs $(Z,W)$ considered by \compfinder{}, each of which can produce up to $|V|$ distinct transit clusters (the maximum amount of components in any induced subgraph), which leads to an upper bound of $|V|^3$ transit components. However, we can find the smallest possible upper bound. First, we present a utility lemma for counting transit components.

\begin{restatable}{lem}{countingcomponents}
\label{lem:counting_components}
Let $T$ be a transit component of a DAG $\sG = (V,E)$ and let $G' = (V', E')$ be the induced graph of the clustering with $t$ representing the set $T$. If there does not exist a transit component $S$ of $\sG$ such that $T \cap S \neq \emptyset$ and $T \setminus S \neq \emptyset$, then
$|\sT_\sG| = |\sT_{\sG'}| + |\sT_{\sG[T]}|$.
\end{restatable}

In other words, Lemma~\ref{lem:counting_components} states that if there exists a transit component $T$ such that no other transit component partially intersects it, then the number of transit components in the original graph $\sG$ is the sum of the number of transit components in the graph induced by clustering $T$ and the number of transit components in the subgraph induced by $T$.

\begin{restatable}{theorem}{componentamount}
\label{th:component_amount}
Let $\sG = (V,E)$ be a DAG. Then $|\sT_{\sG}| \leq \frac{|V|(|V| + 1)}2 - 1$.
\end{restatable}

In contrast, the upper bound for the number of transit clusters grows exponentially as the number of vertices in the graph grows, hence ruling out an efficient algorithm for listing all transit clusters. Figure~\ref{fig:numberoftransitclusters} shows an example where any combination of transit components $\{r_i,e_i\}$, $i=1,\ldots,k$ is a transit cluster and the number of non-singleton transit clusters is $2^k-1$.



Fortunately, by carefully combining transit components, we can construct an algorithm that lists all transit clusters with a polynomial delay. Algorithm~\ref{alg:find_clusters} (\clustfinder{}) attempts to recursively combine transit components into transit clusters using Corollary~\ref{th:union} to detect which unions are valid. Theorem~\ref{thm:cluster_is_component_union} guarantees that all transit clusters will be found by this approach.

\begin{algorithm}[!ht]
  \begin{algorithmic}[1]
  \Function{FindTrClust}{$\sG, \sT_{\sG|R}$}
  \State $\sA \gets \sT_{\sG|R}$
  \State $\sB \gets \sT_{\sG|R}$
  \ForAll{$T \in \sT_{\sG|R}$} \label{line:loop_components}
    \State $\sB \gets \sB \setminus \{T\}$
    \State $\sA \gets \sA \cup \Call{ExpandClust}{T, \sA, \sB, \sG}$ \label{line:recursion_start}
  \EndFor
  \State \Return $\mathcal{A}$ \label{line:return_clusters}
  \EndFunction
  \end{algorithmic}
  \begin{algorithmic}[1]
  \Function{ExpandClust}{$T, \sA, \sB, \sG$}
  \State $\sB' \gets \sB$
  \ForAll{$S \in \sB$} \label{line:loop_remaining_components}
    \State $\sB' \gets \sB' \setminus \{S\}$
    \If{$S \cup T \not\in \sA$ \textbf{and} Corollary~\ref{th:union} holds for $S$ and $T$} \label{line:theorem_condition}
      \State $\sA \gets \sA \cup \{S \cup T\} \cup \Call{ExpandClust}{S \cup T, \sA, \sB', \sG}$ \label{line:valid_union_recursion}
    \EndIf
  \EndFor
  \State \Return $\mathcal{A}$ \label{line:return_from_recursion}
  \EndFunction
  \end{algorithmic}
  \caption{An algorithm for finding all (restricted) transit clusters of a DAG. The inputs are a DAG $\sG = (V,E)$ and the set of all transit components $\sT_{\sG|R}$ with respect to $R \subset V$. Returns the set of all restricted transit clusters in $\mathcal{G}$ with respect to $R$. The subroutine \Call{ExpandClust}{} is used to recursively construct the clusters.}
  \label{alg:find_clusters}
\end{algorithm}

We begin by showing that \clustfinder{} always terminates.

\begin{restatable}{lem}{clustterminates}
\label{lem:clust_terminaes}
\clustfinder{} always terminates for valid inputs $\sG$ and $\sT_{\sG|R}$.
\end{restatable}

Lemma~\ref{lem:clust_terminaes} guarantees that the output of \clustfinder{} is well-defined. Next, we show that the output of the algorithm is a set of transit clusters.
\begin{restatable}[Soundness of \clustfinder{}]{theorem}{clustsound}
Let $\sG = (V,E)$ be a DAG, $R \subseteq V$, and $\sA = \clustfinder(\sG, \sT_{\sG|R})$, then $\sA \subseteq \Upsilon_{\sG|R}.$
\end{restatable}

For the inverse, we show that any transit cluster is found by \clustfinder{}.
\begin{restatable}[Completeness of \clustfinder{}]{theorem}{clustcomplete} Let $\sG = (V,E)$ be a DAG, $R \subseteq V$, and $\sA = \clustfinder(\sG, \sT_{\sG|R})$ then $\Upsilon_{\sG|R} \subseteq \sA$.
\end{restatable}

Finally, we prove that all transit clusters of a DAG can be listed with a polynomial delay.
\begin{restatable}{theorem}{clustpolydelay}
\clustfinder{} outputs all restricted transit clusters of a DAG $\sG = (V,E)$ with respect to $R \subseteq V$ with $O\left(|V|^5\right)$ polynomial delay and a $O\left(|V| + |E|\right)$ initialization delay.
\end{restatable}

If \clustfinder{} and \compfinder{} are run in sequence for the same DAG, a dynamic programming approach can be applied to further eliminate the preprocessing delay of \compfinder{} by caching the parent and child sets of the receivers and emitters of each transit component during the operation of \clustfinder{}. We also note that in practice, the worst case performance only occurs in the first iteration of the outermost recursion level, because the number of possible unions of transit components always decreases in both the number of loop iterations and recursion depth. 

Naturally, it is not necessary to obtain all transit components, and the iteration can be stopped for example when a cluster with some desired properties is found. Alternatively, one can consider transit components directly, as they are valid transit clusters themselves, without attempting to find larger transit clusters at all.
Furthermore, we note that it is not necessary to consider restrictions directly on the transit components in form of the set $R$. The same set of transit clusters can also be obtained by first finding the unrestricted transit clusters and then simply discarding those that violate the restrictions. This type of approach can be useful when the possible restrictions are not known beforehand.

\section{Transit Clusters and Causal Inference} \label{sec:transitcausal}

So far, we have considered clustering from a purely graphical point of view. However, in the context of causal inference and structural causal models of \citet{pearl:book2009}, the causal model defines some variables as unobserved background variables and others as observed, which has to be accounted for when constructing transit clusters in causal diagrams. 

Let $W$ be a set of vertices. We denote by $(X_w)_{w \in W}$ a collection of random variables taking values in measurable spaces $(\mathfrak{X}_w)_{w \in W}$. We assume that the measurable spaces are finite-dimensional vector spaces or finite discrete sets. For $A \subseteq W$ let $\mathfrak{X}_A \equiv \times_{a \in A} (\mathfrak{X}_a)$ denote the product space, and $X_A \equiv (X_a)_{a \in A}$ the corresponding random vector. We will use $p(\cdot|\cdot)$ to denote joint distribution, marginal distributions, and conditional distributions of random variables.

To facilitate the concept where a single vertex represents the entire cluster in the induced graph of the clustering, we adopt a definition of a causal model that explicitly makes it possible for a single vertex of the causal diagram to correspond to a multivariate random variable. Assume that a DAG $\sG=(V,E)$ is clustered as $\sG'=(V',E')$ and let $W' \subset V'$ be the clustering equivalent set of $W \subset V$ resulting from a clustering of a transit cluster $T$. Suppose now that for each $v \in V$, there is a corresponding random variable $X_v$. We can now define clustering equivalent random variables as follows: for any $w' \in W' \setminus \{t\}$, $X_{w'} = X_w$, and $X_{t} = (X_w)_{w \in W \cap T}$, i.e., the random variables corresponding to the clustered vertices are combined into a new random variable and the variables unrelated to the transit cluster remain unchanged. Thus, for any functional $g$ of the joint distribution it holds that $g(p(x_W)) = g(p(x_{W'}))$.



We define structural causal models analogously to \citet{pearl:book2009} using our notation.
\begin{definition}[Causal model] \label{def:causalmodel} 
A causal model $\sM$ is a tuple $(X_V,X_U,\sF,p)$, where 
\begin{itemize}
  \item $X_V$ is an observed random vector indexed by the set $V$.
  \item $X_U$ is an unobserved random vector indexed by the set $U$.
  \item $\sF$ is a collection of functions $(f_v)_{v \in V}$ such that each $f_v$ is a mapping from $\mathfrak{X}_{U \cup (V \setminus \{v\})}$ to $\mathfrak{X}_v$ and such that $\sF$ forms a mapping from $\mathfrak{X}_U$ to $\mathfrak{X}_V$. Symbolically, the set of equations $\sF$ can be represented by writing $X_v = f_v(X_{Pa(v)}, X_{U(v)})$, where $X_{Pa(v)} \subseteq X_V$ is the unique minimal set of observed variables sufficient for representing $f_v$. Likewise, $X_{U(v)} \subseteq X_U$ stands for the unique minimal set of unobserved variables sufficient for representing $f_v$.
  \item $p$ is the joint probability distribution of $X_U$.
\end{itemize}
\end{definition}

Each causal model $\sM$ can be associated with a directed graph $\sG(\sM)$ where the vertices correspond to the sets $V$ and $U$ and directed edges point from members of $Pa(v)$ and $U(v)$ to $v$. We refer to this graph as the causal diagram.  We consider recursive semi-Markovian causal models in this paper, meaning that $\sG(\sM)$ is a DAG and each $u \in U$ has at most two children in $\sG(\sM)$. For simplicity, we assume that noise terms, i.e., vertices of unobserved variables with only one child, are always clustered together with their children when clustering is carried out. This makes it unnecessary to include such unobserved variables when drawing causal diagrams.

We make a distinction between vertices of a DAG and the random variables that they represent in the causal model and equate them only when it is suitable to do so. In figures that depict DAGs that are causal diagrams, we draw vertices $V$ that relate to observed variables as circles, and vertices $U$ that relate to unobserved variables as squares. The $\doo(x_A)$ operator denotes that the variables $X_A$ are assigned fixed values $x_a$ irrespective of their parents in the causal diagram. 

We define the identifiability of a causal effect as follows.
\begin{definition}[Identifiability]
An interventional distribution (causal effect) $p(x_A|\doo(x_B))$ is identifiable from $p(x_V)$ in a causal diagram $\sG = (V,E)$ if it is uniquely computable from $p(x_V)$ in every causal model that has the causal diagram $\sG$.
\end{definition}

Causal effect identification is a well-known problem in causal inference, and a complete algorithm exists for the problem of determining whether $p(x_A\cond \doo(x_B))$ is identifiable in a causal diagram $\sG$ from the observational distribution $p(x_V)$ \citep{shpitser2006,tikka2017b}. An important graphical structure related to identifiability is a confounded component, or a c-component for short (sometimes also called a ``district''). C-components are typically defined for semi-Markovian causal models, that is, models where unobserved variables have exactly two children and they are represented graphically by bidirected edges instead of including the corresponding variables explicitly in the graph. Because transit clusters can contain vertices that represent unobserved variables, we provide an equivalent definition of c-components for causal diagrams where unobserved variables are present.

\begin{definition}[c-component] \label{def:c-component}
Let $\sG = (V \cup U, E)$ be a causal diagram of a causal model $\sM$. If there exists a path between every pair of vertices $i,j \in V$ such that every vertex on the path that is a member of $V$ is a collider with the exception of $i$ and $j$, and the path contains at least one vertex that is a member of $U$, then $\sG$ is a c-component.
\end{definition}

When c-components are defined as in Definition~\ref{def:c-component}, we may define maximal c-components, c-trees, c-forests and hedges analogously they are defined by \citet{shpitser2006} (see Appendix~\ref{app:hedges} for details). Any causal diagram that is not a c-component can be uniquely partitioned into a set of its maximal c-components. A causal effect $p(x_A \cond \doo(x_B))$ is identifiable from $p(x_V)$ in $\sG = (V \cup U,E)$ if and only if $\sG$ does not contain a hedge for any $p(x_{A'} \cond \doo(x_{B'}))$ in $\sG$, where $A' \subseteq A$ and $B' \subseteq B$. The existence of a hedge means that the graph has two c-components that fulfill specific graphical conditions, and that can be used to construct two causal models that agree on $p(x_V)$ but disagree on $p(x_A \cond \doo(x_B))$.

Identifiability may not be preserved by arbitrary clusters of vertices in the causal diagram, meaning that a causal effect may be identifiable in the original graph but not in the graph induced by the cluster or vice versa. Fortunately, transit clusters can be proven to preserve identifiability under specific conditions. The first condition requires that the emitters and the parents of receivers are observed:
\begin{definition} \label{def:plaintransitcluster}
A transit cluster $T$ in a DAG $\sG = (V \cup U,E)$ of a causal model $\sM$ is \emph{plain} if $\Pa^*(\rec(T)) \cup \emi(T) \subseteq V$.
\end{definition}
In a plain transit cluster, any latent confounders will always be clustered together with their children. The second condition requires that the entire transit cluster belongs to the same c-component: 
\begin{definition} \label{def:congestedtransitcluster}
A transit cluster $T$ in a DAG $\sG = (V \cup U,E)$ of a causal model $\sM$ is \emph{congested} if all members of $T$ belong to the same c-component in $\sG$ and  $\emi(T) \subseteq V$.
\end{definition}
In essence, plain and congested transit clusters do not change the c-components of the causal diagram. Thus, identifiability is preserved for plain and congested transit clusters.
\begin{restatable}{theorem}{identifiability}
\label{th:identifiability}
Let $\sG = (V \cup U,E)$ be a DAG of a causal model $\sM$ and let $A$ and $B$ be disjoint subsets of $V$. Let $T = \{t_1, \ldots, t_n\}$ be a restricted transit cluster with respect to $(V \cup U) \setminus (A \cup B)$ in $\sG$, and let $\sG' = (V' \cup U',E')$ be the induced graph of the cluster with vertex $t' \in V'$ as the representative of $T$. 
If $T$ is plain or congested, then $p(x_A \cond \doo(x_B))$ is identifiable from $p(x_V)$ in $\sG$ precisely when it is identifiable from $p(x_{V'})$ in $\sG'$. 
\end{restatable}

Figure~\ref{fig:nonidtransitcluster} demonstrates that in general identifiability may be lost in clustering even if the cluster is a transit cluster. In these examples, receiver $r$ and emitter $e$ form a transit cluster that is neither plain nor congested because $r$ has unobserved variables as parents but $r$ and $e$ do not belong to the same c-component. Consequently, Theorem~\ref{th:identifiability} does not apply for these transit clusters.
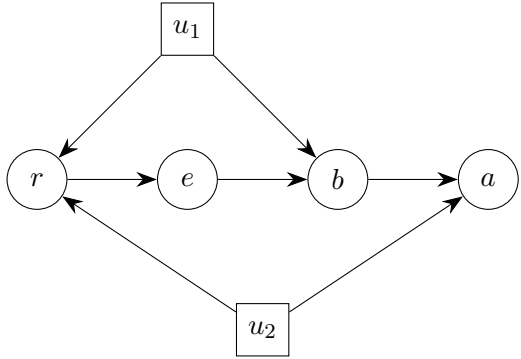
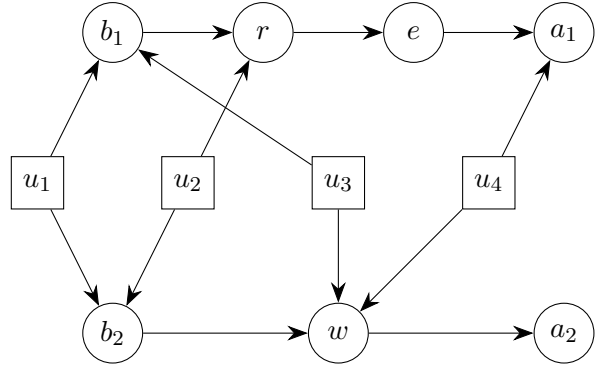
\begin{figure}[ht]
\begin{center}
\begin{subfigure}[t]{0.35\textwidth}
\centering
\begin{tikzpicture}
\node [obs = {a}] at (2,0) {};
\node [obs = {b}] at (0,0) {};
\node [obs = {e}] at (-2,0) {};
\node [obs = {r}] at (-4,0) {};
\node [lat = {u_1}] at (-2,2) {};
\node [lat = {u_2}] at (-1,-2) {};
\path [->] (r) edge (e);
\path [->] (b) edge (a);
\path [->] (e) edge (b);
\path [->] (u_1) edge (b);
\path [->] (u_1) edge (r);
\path [->] (u_2) edge (a);
\path [->] (u_2) edge (r);
\end{tikzpicture}
\caption{Query $p(x_a \cond \doo(x_b))$ is identified from $p(x_r,x_e,x_b,x_a)$.} 
\end{subfigure}
\hfill
\begin{subfigure}[t]{0.5\textwidth}
\centering
\begin{tikzpicture}
\node [obs = {b_1}] at (0,4) {};
\node [obs = {r}] at (2,4) {};
\node [obs = {e}] at (4,4) {};
\node [obs = {a_1}] at (6,4) {};
\node [obs = {b_2}] at (0,0) {};
\node [obs = {w}] at (3,0) {};
\node [obs = {a_2}] at (6,0) {};
\node [lat = {u_1}] at (-1,2) {};
\node [lat = {u_2}] at (1,2) {};
\node [lat = {u_3}] at (3,2) {};
\node [lat = {u_4}] at (5,2) {};
\path [->] (b_1) edge (r);
\path [->] (r) edge (e);
\path [->] (e) edge (a_1);
\path [->] (b_2) edge (w);
\path [->] (w) edge (a_2);
\path [->] (u_1) edge (b_1);
\path [->] (u_1) edge (b_2);
\path [->] (u_2) edge (r);
\path [->] (u_2) edge (b_2);
\path [->] (u_3) edge (b_1);
\path [->] (u_3) edge (w);
\path [->] (u_4) edge (a_1);
\path [->] (u_4) edge (w);
\end{tikzpicture}
\caption{Query $p(x_{a_1},x_{a_2} \cond \doo(x_{b_1},x_{b_2}))$ is identified from $p(x_{b_1},x_{b_2},x_r,x_e,x_w,x_{a_1},x_{a_2})$.}
\end{subfigure}
%

%

\caption{Two examples where identifiability is lost when transit cluster $T=\{r,e\}$ is replaced by a single vertex. The vertices of unobserved variables are denoted by squares.} \label{fig:nonidtransitcluster}
\end{center}
\end{figure}

\section{Robustness of Structural Assumptions} \label{sec:robustness}

Consider now the top-down causal modeling where a cluster $T$ represents observed variables $(X_{t_1},\ldots, X_{t_n})$ and an unspecified number of unobserved background variables. The cluster is represented by a single node $t$ in a DAG $\sG^\prime$ but the causal structure inside the cluster has not been specified. Assume that a causal effect $p(x_A|\doo(x_B))$, where $T \cap (A \cup B) = \emptyset$, is identifiable from $p(x_{V'})$ under the structural assumptions coded in a causal diagram $\sG^\prime$ and $g(p(x_{V'}))$ is an identifying functional for $p(x_A|\doo(x_B))$. 
We are interested in characterizing the internal structure of cluster $T$ for which we can guarantee that $g(p(x_V))$ obtained from $g(p(x_{V}'))$ by explicitly replacing $x_t$ by its observed components $(x_{t_1},\ldots, x_{t_n})$, is an identifying functional for $p(x_A|\doo(x_B))$ in $\sG$. We will show that a plain or congested transit cluster fulfills this requirement of structural robustness. 

We start with DAG $\sG^\prime$ where the single node transit cluster $\{t\}$ represents the random variables $X_{t_1},\ldots, X_{t_n},X_{u_1},\ldots, X_{u_m}$, and the number of unobserved variables $m$ has been chosen arbitrarily. We apply the peripheral extension of Definition~\ref{def:peripheral} until all variables $X_{t_1},\ldots, X_{t_n},X_{u_1},\ldots, X_{u_m} $ of the cluster are explicitly presented as vertices $\{t_1, \ldots, t_n,u_1,\ldots,u_m\}$ of graph $\sG$.  Finally, we state conditions that the identifying functional remains valid. 

\begin{restatable}{theorem}{extensionid} \label{th:extension_identifiability}
Let $X_V$ be a vector of observed random variables, $X_U$ a vector of unobserved random variables, and $\sG' = (V' \cup U',E')$ the causal diagram of a causal model $\sM'$ where vertex $t \in V'$ represents set $T = \{t_1, \ldots, t_n,u_1,\ldots,u_m\}$ in $\sG'$,  $t_1, \ldots, t_n \in V$,  $u_1,\ldots, u_m \in U$, and $v \in V' \setminus \{t\}$ implies $v \in V$. Let $\sG = (V \cup U, E)$ be a DAG obtained from $\sG^\prime$ by applying a series of peripheral extensions to vertex $t$ such a way that $\sG$ is a causal diagram. If $T$ is a plain or congested transit cluster in $\sG$, the following holds for disjoint subsets $A$ and $B$ of $V'$ such that $T \cap (A \cup B) = \emptyset$.
\begin{enumerate}
\item  Causal effect $p(x_A \cond \doo(x_B))$ is identifiable from $p(x_{V})$ in $\sG$ exactly when it is identifiable from $p(x_{V'})$ in $\sG'$. 
\item If $g(p(x_{V'}))$ is an identifying functional for $p(x_A \cond \doo(x_B))$ in $\sG'$, 
 it is also an identifying functional for $p(x_A \cond \doo(x_B))$ in $\sG$.
\end{enumerate}
\end{restatable}

\section{Use cases and illustrations} \label{sec:illustration}

Transit clusters can be applied in various ways. Here we demonstrate their use in reducing the size of a causal diagram, simplification of identifying functionals, speeding up identification algorithms, and top-down causal modeling.

\subsection{Reducing the size of a causal diagram}

As an example of simplification of the causal diagrams and identifying functionals, as well as the robustness of causal effect estimation, we consider a graph related to the Sangiovese grapes studied earlier as a conditional linear Gaussian network by \citet{Magrini2017}, where the interest was in the effect of various treatments $X_b$ on the mean weight of grapes $X_a$. In addition to $a$ and $b$, the graph contains vertices $z_1,\ldots, z_{13}$ which are related to different characteristics of the soil, grape plants and must \citep{Magrini2017}. Compared to the original graph, we added a latent confounder $X_u$ between the treatment variable $X_b$ and the mean weight of grapes $X_a$ for illustrative purposes. We do not present the causal diagram graphically as its large number of vertices (16) and edges (57) makes it difficult to visualize clearly.

Applying algorithm \ref{alg:find_clusters} gives us transit cluster $T = \{z_1,\ldots,z_{13}\}$, with variables $r = \{z_1, z_2, z_4, z_{10}\}$ as receivers and $e = \{z_1, z_3, z_4, z_6, z_7, z_8,z_{10},z_{12},z_{13}\}$ as emitters (with variables $z_5$, $z_9$, and $z_{11}$ being neither). This leads to a simplified graph in Figure \ref{fig:sangiovese}.

\begin{figure}[h]
	\begin{center}
		\begin{tikzpicture}
		\node [obs = {a}] at (2,0) {};
		\node [obs = {b}] at (-2,0) {};
		\node [obs = {t}] at (0,0) {};
		\node [lat = {u}] at (0,2) {};
		\path [->] (b) edge (t);
		\path [->] (t) edge (a);
		\path [->] (u) edge (a);
		\path [->] (u) edge (b);
		\end{tikzpicture}
		\caption{Induced graph of the Sangiovese graph where the transit cluster $T = \{z_1,\ldots,z_{13}\}$ is replaced with single vertex $t$.}
		\label{fig:sangiovese}
	\end{center}
\end{figure}
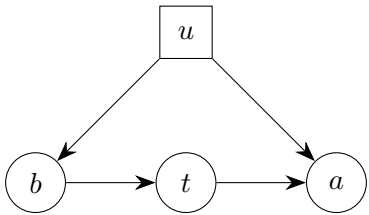

\subsection{Simplification of identifying functionals}

The application of the ID-algorithm \citep{shpitser2006} to original graph related to the Sangiovese grapes leads to long and complicated identifying functional:
\begin{equation}\label{eq:sangiovese_nocluster}
\begin{aligned}
p(x_a \cond \doo(x_b)) &=
\sum_{x_{z_1},\ldots,x_{z_{13}}}\left[\vphantom{\left(\sum_{x_b}\right)}\right. 
p(x_{z_{13}}|x_b,x_{z_1},\ldots,x_{z_{12}})p(x_{z_{12}}|x_b,x_{z_1},\ldots, x_{z_8}, x_{z_{10}}, x_{z_{11}}) \\
&\times p(x_{z_{11}}|x_b,x_{z_1},\ldots, x_{z_8}, x_{z_{10}})p(x_{z_{10}}|x_b,x_{z_1},\ldots, x_{z_8})p(x_{z_9}|x_b,x_{z_1},\ldots, x_{z_8})\\
&\times p(x_{z_8}|x_b,x_{z_1},\ldots, x_{z_7}) p(x_{z_7}|x_b,x_{z_1},\ldots, x_{z_6}) p(x_{z_6}|x_b,x_{z_1},x_{z_2},x_{z_3},x_{z_5})\\
&\times p(x_{z_5}|x_b,x_{z_1},x_{z_2},x_{z_3})p(x_{z_4}|x_b,x_{z_1})p(x_{z_3}|x_b,x_{z_1},x_{z_2})p(x_{z_2}|x_b,x_{z_1})p(x_{z_1}|x_b)  \\
&\times \left. \left(\sum_{x'_b}p(x_a|x'_b,x_{z_1},\ldots,x_{z_{13}})p(x'_b)\right)\right]
\end{aligned}
\end{equation}
On the contrary, the application of the ID-algorithm to the clustered graph of Figure~\ref{fig:sangiovese} leads to identifying functional of form
\begin{equation}\label{eq:sangiovese_cluster}
\begin{aligned}
p(x_a \cond \doo(x_b)) &= \sum_{x_t}p(x_t | x_b)\left(\sum_{x'_b}p(x_a|x'_b,x_t)p(x'_b)\right),
\end{aligned}
\end{equation}
i.e. front-door adjustment. This enables us to model $p(x_t | x_b)$ in an arbitrary, but consistent manner without making specific claims about the internal structure of $T$.

\subsection{Speeding up identification algorithms}

Identifying functionals obtained from the application of ID-algorithm or general do-calculus are often unnecessarily complex and could be further simplified \citep{tikka2017simplifying}. This can allow easier interpretation of the identifying functional and more efficient estimation of the causal effect. The R \citep{R} package causaleffect \citep{tikka2017b} implements an automatic simplification algorithm for this task, however, the algorithm can be slow in case of large graphs. Alternatively if we can first simplify the graph by clustering, we can reduce the computational 
burden of both the identification algorithm as well as subsequent simplification algorithm. For example, in case of the Sangiovese graph, the causaleffect package returns \eqref{eq:sangiovese_nocluster} in 0.1 seconds with simplification option disabled and 132 seconds with simplification enabled (which in this case does not lead to simpler equation) on a standard laptop. On the other hand, running the clustering algorithm and subsequent identification (which returns \eqref{eq:sangiovese_cluster}) takes only 0.5 seconds. Importantly, the simplification algorithm of \citet{tikka2017simplifying} is NP-hard, and thus it may be possible to obtain simpler identifying functionals using transit clusters in scenarios where direct simplification is infeasible. The code for this benchmark is also available in the GitHub repository.

\subsection{Top-down causal modeling}

As an example of peripheral extension and the robustness of the estimation strategies, consider a causal graph shown in Figure~\ref{fig:fsd}, studied earlier by \citet{helske2021}, where the interest is in the causal effect of the education level $X_e$ on income $X_i$. Variable $X_s$ measures general language skills on Illinois Test of Psycholinguistic Abilities (ITPA), which is a composite of 12 subtests. Thus instead of vertex $s$ representing a single composite variable $X_s$ in Figure~\ref{fig:fsd}, we can by the peripheral extension (\ref{def:peripheral}) treat it as a transit cluster $T = \{s_1,\ldots, s_{12}\}$ (with $s_i$ corresponding to the subtest $i$) without affecting the identifiability of the causal effect $p(x_i \cond \doo(x_e))$. While the causal effect estimates can depend on whether we use $X_s$ or $X_T$ in the modelling, the obtained estimator is robust to these changes in a sense that the methodology of \citet{helske2021} can be used to estimate the effect in both cases.

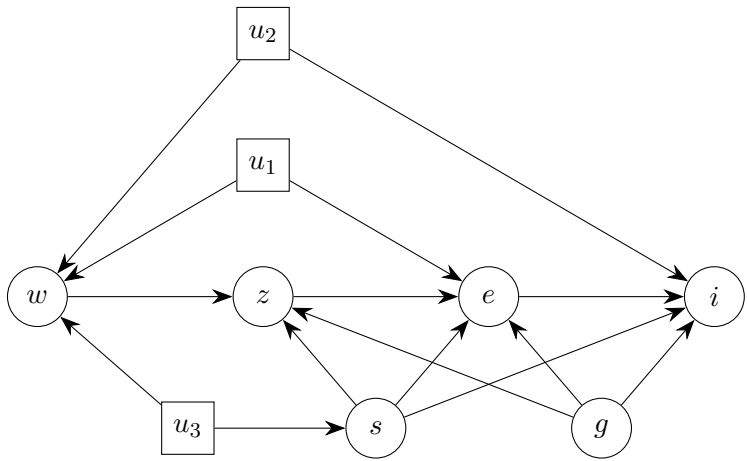
\begin{figure}
  \begin{center}
\begin{tikzpicture}
\node [obs = {g}] at (1.5,-1.75) {};
\node [obs = {s}] at (-1.5,-1.75) {};
\node [obs = {i}] at (3,0) {};
\node [obs = {e}] at (0,0) {};
\node [obs = {z}] at (-3,0) {};
\node [obs = {w}] at (-6,0) {};
\node [lat = {u_1}] at (-3,1.75) {};
\node [lat = {u_2}] at (-3,3.5) {};
\node [lat = {u_3}] at (-4,-1.75) {};
\path [->] (s) edge (i);
\path [->] (w) edge (z);
\path [->] (s) edge (e);
\path [->] (s) edge (z);
\path [->] (e) edge (i);
\path [->] (g) edge (i);
\path [->] (g) edge (e);
\path [->] (g) edge (z);
\path [->] (z) edge (e);
\path [->] (u_1) edge (e);
\path [->] (u_1) edge (w);
\path [->] (u_2) edge (i);
\path [->] (u_2) edge (w);
\path [->] (u_3) edge (s);
\path [->] (u_3) edge (w);
\end{tikzpicture}

\caption{Causal diagram representing the effect of education level $X_e$ on income $X_i$. Other variables represented in graph are gender $X_g$, score on Illinois Test of Psycholinguistic Abilities (ITPA) $X_s$, socioeconomic status of the parents $X_w$, and the grade point average $X_z$ at the end of primary school. Variables $X_{u_1}, X_{u_2}, X_{u_3}$ are unobserved.}
\label{fig:fsd}
\end{center}
\end{figure}

As another example of peripheral extension, we consider an epidemiological application studied earlier by \citet{karvanen2020search}. The question of interest is the causal effect of salt-adding behavior on the salt intake. The high salt intake is one of the causes of hypertension \citep{he2013effect}. 

The example is based on the National Health and Nutrition Examination Survey (NHANES, \url{https://wwwn.cdc.gov/nchs/nhanes/}) 2015--2016 that is an observational study on the health and nutritional status of adults and children in the United States.
The NHANES variables are already divided into categories by their content and the type of the data. In the top-down modeling, these categories may correspond to transit clusters in the causal diagram. An example of a causal diagram constructed by this approach is shown in Figure~\ref{fig:salt}. The peripheral extension (Definition~\ref{def:peripheral}) can be used to extend the clusters. For instance, the cluster represented by the vertex $a$, ``Salt-adding behavior'', may consists of the following variables measured in NHANES:
\begin{enumerate}
 \item \textit{How often do you add ordinary salt to your food at the table?} (Rarely 0, Occasionally 1, Very often 2)
 \item \textit{Did you add any salt to your food at the table yesterday?} (No 0, Yes 1), and
 \item \textit{How often is ordinary salt or seasoned salt added in cooking or preparing foods in your household?} (Never 0, Rarely 1, Occasionally 2, Very often 3).
\end{enumerate}
and the variables of the cluster represented by vertex $b$, ``Diet behavior'', may include
\begin{enumerate}
\item \textit{Are you on low salt/low sodium diet?}
\item \textit{Are you on other special diet?} (several options)
\item \textit{Number of meals not home prepared during the past 7 days}
\item \textit{Number of meals from fast food or pizza place during the past 7 days}
\end{enumerate}

The use of transit clusters provides a formal justification for the top-down modeling.
Especially, Theorem~\ref{th:extension_identifiability} states sufficient conditions for the validity of conclusions made with the clustered graph. 

\begin{figure}
\begin{center}
\begin{tikzpicture}[xscale=2.5,yscale=2.0]
\node [obs = {o}] at (0,0) {};
\node [obs = {d}] at (1,1) {};
\node [obs = {i}] at (1,-1) {};
\node [obs = {a}] at (2,1) {};
\node [obs = {b}] at (2,-1) {};
\node [obs = {s}] at (3,0) {};
\node [lat = {u_1}] at (0.4,-1.6) {};
\node [lat = {u_2}] at (4,0) {};

\path [->] (d) edge (o);
\path [->] (d) edge (i);
\path [->] (d) edge (a);
\path [->] (d) edge (b);
\path [->] (d) edge (s);

\path [->] (o) edge (i);
\path [->] (o) edge (a);
\path [->] (o) edge (b);
\path [->] (o) edge (s);

\path [->] (i) edge (a);
\path [->] (i) edge (b);
\path [->] (i) edge (s);

\path [->] (b) edge (a);
\path [->] (b) edge (s);
\path [->] (a) edge (s);

\path [->] (u_1) edge (o);
\path [->] (u_1) edge (i);
\path [->] (u_1) edge (b);
\path [->] (u_2) edge (a);
\path [->] (u_2) edge (b);
\end{tikzpicture}
\caption{Causal model for the salt intake example. The vertices represent the following transit clusters of variables: Salt-adding behavior $X_a$ (represented by a single vertex $a$), salt intake $X_s$, diet behavior $X_b$, demographic variables  $X_d$, occupation $X_o$, and income $X_i$.} \label{fig:salt}
\end{center}
\end{figure}
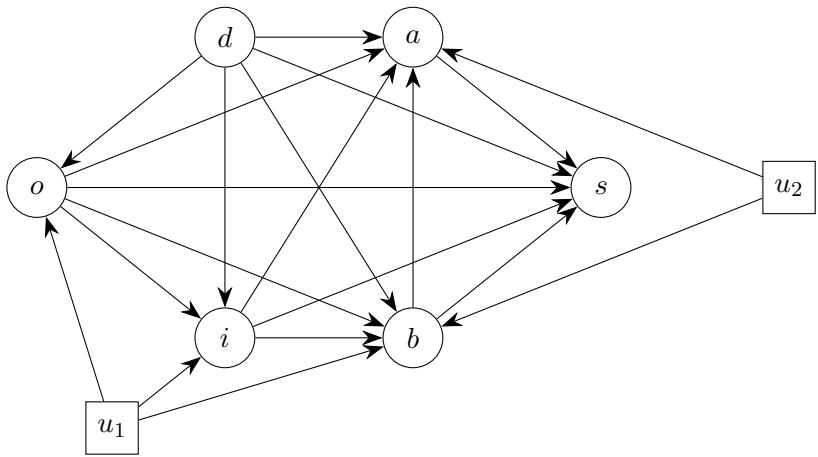

\section{Discussion} \label{sec:discussion}
We have considered clustering from two starting points. First, we started with an unclustered DAG that may have a large number of vertices and proposed algorithms for finding transit components and transit clusters, allowing us to simplify the representation of the DAG and the obtained identifying functional. Furthermore, we provided sufficient conditions for non-identifiability in a clustered DAG to imply non-identifiability in the original DAG. Second, we started with a clustered DAG where a single vertex represents a group of variables and presented the peripheral extension, a procedure for constructing all transit clusters that are compatible with the clustered DAG. We showed that an identifying functional for a causal effect in the clustered DAG remains valid in DAGs obtained via peripheral extension. 

A transit cluster was deliberately defined for a DAG without any reference to a causal model. This allows us to cluster vertices even before it is known which data will be available. The division into observed and unobserved variables is however hard-coded into the definition of a structural causal model where an unobserved variable cannot have parents. This restriction is taken into account in Theorems~\ref{th:identifiability} and \ref{th:extension_identifiability}.

The DAG-based definition of a transit cluster makes it possible to apply a workflow where Algorithm~\ref{alg:find_components} is first run for the whole graph in order to find all transit components. Restrictions may then be applied to these transit components before transit clusters are constructed by Algorithm~\ref{alg:find_clusters}. The same transit components can be re-used when the causal effect in the focus is changed to a new one that implies different restrictions for the transit clusters.


The examples presented in Section~\ref{sec:illustration} illustrate the use of transit clusters in reducing the size of a causal diagram, simplifying identifying functionals, and speeding up identification algorithms. 
The identifying functional defined using the representative vertex in place of transit cluster allows a researcher to focus on the overall structure of the functional when choosing suitable estimation methods for the causal effect. Transit clusters also provide a justification for the top-down causal modeling. 

In addition to the use cases considered,
clustering could be beneficial also in causal discovery \citep{spirtes2000,spirtes2001}. If a set of variables can be assumed to form a transit cluster, we may, at least theoretically, use any single variable of the set as representative of the whole cluster when considering whether the cluster and a variable outside the cluster are d-separated. In general, causal discovery methods can construct the underlying DAG only up to an equivalence class and additional challenges with finite samples may occur due to a variety of reasons, such as measurement error \citep{zhang2017}, selection bias \citep{zhang2016}, or missing data \citep{tu2019}. The assumption on a transit cluster could in some cases provide the information needed to reduce these ambiguities.

In future work, we would like to extend the results of Sections~\ref{sec:transitcausal} and \ref{sec:robustness} to more general identifiability problems with multiple data sources consisting of a mix of observational and interventional distributions. We hypothesize, that at least plain transit clusters can be used to retain identifiability in more complex settings. It may also be possible to extend the definition of transit clusters to graphs where the direction of some edges is unknown.

\acks{This work was supported by Academy of Finland grant numbers 311877 and 331817.}

\appendix
\section{Proofs} \label{app:basic_proofs}
We restate and prove all results of the paper.

\isdag*
\begin{proof}
Let $t$ be the single vertex in $\sG^\prime$ that corresponds to $T$ in $\sG$. We show that if there exists a directed path from $v_1$ to $v_2$ in $\sG^\prime$, there cannot also exist a directed path from $v_2$ to $v_1$. Assume first that both directed paths exist and neither of them contains $t$. This is a contradiction because then both paths would exist in a DAG $\sG$ as well. Next, assume without loss of generality that the directed path from $v_1$ to $v_2$ contains $t$. It follows that $\sG$ has a directed path $v_1 \rightarrow \ldots \rightarrow r \rightarrow \ldots \rightarrow e \rightarrow \ldots \rightarrow v_2$, where the existence of vertices $r \in \rec(T)$ and $e \in \emi(T)$ is guaranteed by the definition of transit cluster. Similarly, if the directed path from $v_2$ to $v_1$ contains $t$, there would a directed path from $v_2$ to $v_1$, which together with directed path from  $v_1$ to $v_2$ would create a cycle in $\sG$. If the directed path from $v_2$ to $v_1$ does not contain $t$, it will exist also in $\sG$ and form a cycle in $\sG$. We conclude that $\sG^\prime$ cannot have cycles and is thus a DAG.
\end{proof}
\reem*
\begin{proof}
Suppose instead that there exists $t \in T$ such that $t \not\in S$ and that $t \not\in \rec(T) \cup \emi(T)$. By condition~\ref{itm:no_ronsy_allowed}, $t$ is connected to a receiver or an emitter in $\sG[T^=]$. Assume first that $t$ is connected to receiver $r$ in $\sG[T^=]$. As $\rec(T)=\rec(S)$, $r$ is also a receiver for $S$. Follow the path from $r$ to $t$ and let $s$ be the last vertex that belongs $S$ and $q$ the next vertex that does not belong to $S$. If there is an edge $s \rightarrow q$, $s$ is an emitter for $S$ and if there is an edge $q \rightarrow s$, $s$ is a receiver for $S$. As $\rec(T)=\rec(S)$ and $\emi(T)=\emi(S)$, $s$ is also a receiver or an emitter for $T$. This leads to a contradiction because by definition the edges incoming to receivers and outgoing from emitters are cut in $\sG[T^=]$ and the path between $t$ and $r$ cannot exist. The case where $t$ is connected to an emitter in $\sG[T^=]$ proceeds analogously.
\end{proof}
\stilltransitcluster*
\begin{proof}
As $S$ and $T$ are disjoint, the vertices of $S$ as well as the edges between vertices of $S$ are unaffected by the clustering of $T$. It follows that $S^\prime = S$, where $S'$ is the clustering equivalent set of $S$. We will show that $\rec_{\sG^\prime}(S^\prime)=\rec_{\sG}(S)$, $\emi_{\sG^\prime}(S^\prime)=\emi_{\sG}(S)$, and $S^\prime$ fulfills the conditions of Definition~\ref{def:transitcluster} both in $\sG$ and $\sG^\prime$. As the edges between $S$ and $V \setminus (S \cup T)$ are unaffected by the clustering, it suffices to consider only edges between $S$ and $T$. If a parent of $\rec_{\sG}(S)$ belongs to $T$, condition~\ref{itm:parents} applied to $S$ in $\sG$ guarantees that vertex $t$ will be a parent of all vertices in $\rec_{\sG}(S)$ in $\sG^\prime$. If $t$ is a parent of $\rec_{\sG^\prime}(S^\prime)$, condition~\ref{itm:parents} applied to $S^\prime$ in $\sG^\prime$ guarantees that any member of $T$ that is a parent of a receiver in $\sG$ will be a parent of all vertices in $\rec_{\sG}(S)$ in $\sG$. Similarly, if a child of $\emi_{\sG}(S)$ belongs to $T$, vertex $t$ will be a children of all vertices in  $\emi_{\sG}(S)$ in $\sG^\prime$ and if $t$ is a child of $\emi_{\sG^\prime}(S^\prime)$, any member of $T$ that is a child of an emitter in $\sG$ will be a child of all vertices in $\emi_{\sG}(S)$ in $\sG$. It follows that $\rec_{\sG^\prime}(S^\prime)=\rec_{\sG}(S)$, $\emi_{\sG^\prime}(S^\prime)=\emi_{\sG}(S)$ and conditions~\ref{itm:parents} and \ref{itm:children} are fulfilled for $S^\prime=S$ both in $\sG$ and $\sG^\prime$. Conditions~\ref{itm:no_ronsy_allowed}, \ref{itm:forallr} and \ref{itm:foralle} are fulfilled as well as they consider only paths inside $S^\prime$.
\end{proof}
\modularity*
\begin{proof}
Denote $Q=T \cup S$ and $Q^\prime=\{t\} \cup S$. 

First, we assume that $\{t\} \cup S$ is transit cluster in $\sG^\prime$ and show that $T \cup S$ is transit cluster in $\sG$.

Condition~\ref{itm:parents}:  We will show for all $r \in \rec[\sG](Q)$  that $\Pa[\sG](r) \setminus Q = \Pa[\sG^\prime](\rec[\sG^\prime](Q^\prime)) \setminus Q^\prime$.
First let $v \notin Q$ to be a parent of receiver $r$ in $\sG$. It follows that in $\sG^{\prime}$, vertex $v$ is a parent of $r$ if $r \in S$ or a parent of $t$ if $r \in T$ because $v \notin Q$ in $\sG$. It follows that $r$ or $t$ is a receiver in $\sG^{\prime}$ and  $v \in \Pa[\sG^\prime](\rec[\sG^\prime](Q^\prime)) \setminus Q^\prime$. 

Now let $v \in \Pa[\sG^\prime](\rec[\sG^\prime](Q^\prime)) \setminus Q^\prime$ and consider two cases: a) If  $v$ is a parent of $t$ in $\sG^\prime$ then $t$ is a receiver in $\sG^\prime$. Further, $v$ is a parent of some $t_i \in T$ in $\sG$ because $t$ is a single vertex corresponding to transit cluster $T$ in $\sG$. It holds $t_i \in \rec[\sG](Q)$ because $v \notin Q$. Condition 1 of transit cluster guarantees that $v$ is also a parent of $r$. 
 b) If $v$ is a parent of $s_i \in S$ in $\sG^{\prime}$ then $v$ is a parent of $s_i \in S$ also in $\sG$ because $v \notin Q^\prime$ in $\sG^{\prime}$. Thus $s_i$ is a receiver in $\sG$ and $v \in \Pa(\rec(Q)_{\sG})_\sG \setminus Q$. 
 
Condition~\ref{itm:children}: The proof is analogous to condition~\ref{itm:parents}. 
 
Condition~\ref{itm:no_ronsy_allowed}: Consider $q_i \in Q$. Assume first that $q_i \in T$. There exist in $\sG^\prime$ vertex $v$ that is a receiver for $Q^\prime$ or an emitter for $Q^\prime$ and is connected to $t$. Applying conditions~\ref{itm:parents}, \ref{itm:children}, \ref{itm:forallr} and \ref{itm:foralle} for transit cluster $T$ guarantees that there a path between $q_i$ and $v$ in $\sG$. If $v \in S$, it is the required receiver or emitter for $Q$. If $v = t$, applying conditions~\ref{itm:parents}, \ref{itm:children}, \ref{itm:forallr} and \ref{itm:foralle} to transit cluster $T$ guarantees that set $T$ has the required receiver or emitter. 

Assume next that $q_i \in S$. There exist in $\sG^\prime$ vertex $v$ that is a receiver for $Q^\prime$ or an emitter for $Q^\prime$ and is connected to $q_i$. If $v \in S$, it satisfies the condition~\ref{itm:no_ronsy_allowed} for $Q$ because applying conditions~\ref{itm:forallr} and \ref{itm:foralle} to transit cluster $T$ guarantees that $t$ can be replaced by a path consisting of vertices in $T$. If $v = t$, applying conditions~\ref{itm:parents}, \ref{itm:children}, \ref{itm:forallr} and \ref{itm:foralle} for transit cluster $T$ guarantees that set $T$ has the required receiver or emitter.

Condition~\ref{itm:forallr}: Assume $\emi[\sG](Q) \neq \emptyset$ and let $r \in \rec[\sG](Q)$. a) If $r \in T$  then $t$ is a receiver in $\sG^\prime$ by the proof of condition~\ref{itm:parents}. It follows that there exists $e^{\prime} \in \emi[\sG^\prime](Q^\prime)$ such that $e^{\prime} \in \De[\sG^\prime](t)$. If $e^{\prime} \neq t$, it directly fulfills condition~\ref{itm:forallr}. If $e^{\prime} = t$, there exists $e \in \emi[\sG](T)$ such that $e \in \De[\sG](r)$, which fulfills condition~\ref{itm:forallr}. b) If $r \in S$ then $r \in \rec[\sG^\prime](Q^\prime)$ and there exists $e \in \emi[\sG^\prime](Q^\prime)$ such that $e \in \De[\sG^\prime](r)$. If $e \in S$, it fulfills condition~\ref{itm:forallr} because $Q^\prime$ is a transit cluster. If $e=t$, there exists an emitter in $T$ that fulfills condition~\ref{itm:forallr} because $T$ is a transit cluster.

Condition~\ref{itm:foralle}: The proof is analogous to condition~\ref{itm:forallr}.

Next, we assume that $T \cup S$ is transit cluster in $\sG$  and show that $\{t\} \cup S$ is transit cluster in $\sG^\prime$. Conditions~\ref{itm:parents} and \ref{itm:children} are already covered above when we showed that $\Pa(\rec(Q)_{\sG})_\sG \setminus Q = \Pa[\sG^\prime](\rec[\sG^\prime](Q^\prime)) \setminus Q^\prime$ and $\Ch(\emi(Q)_{\sG})_\sG \setminus Q = \Ch[\sG^\prime](\emi[\sG^\prime](Q^\prime)) \setminus Q^\prime$.

Condition~\ref{itm:no_ronsy_allowed}: Consider $q_i \in Q^\prime$. Assume first that $q_i = t$. For any  $t_i \in T$ there exist some vertex $v$ that is the required receiver or emitter in $\sG$. By the definition of clustering, $t$ and $v$ are connected in $\sG^\prime$. Assume next that $q_i \in S$. There exist in $\sG$ vertex $v$ that is a receiver for $Q$ or an emitter for $Q$ and is connected to $q_i$. If $v \in S$, it satisfies the condition~\ref{itm:no_ronsy_allowed}. If $v \in T$, vertex $t$ is the required receiver or emitter for $q_i$ in $\sG^\prime$. 

Condition~\ref{itm:forallr}: Assume $\emi[\sG^\prime](Q) \neq \emptyset$ and let $r \in \rec[\sG^\prime](Q^\prime)$. a) If $r = t$, set $\rec_{\sG}(Q) \cap T$ is non-empty and all members of this set fulfill condition~\ref{itm:forallr} in $\sG$. Let there be a path from $r_i \in \rec_{\sG}(Q) \cap T$ to $e \in \emi_{\sG}(Q)$ in $\sG$. It follows that either $e \in T$ and $t$ is the requested emitter in $\sG^\prime$ or $e \in S$ and there is a path from $t$ to $e$ in $\sG^\prime$ and $e$ is the requested emitter.

Condition~\ref{itm:foralle}: The proof is analogous to condition~\ref{itm:forallr}.
\end{proof}
\clusterunion*
\begin{proof}
Consider graph $\sG^\prime$ where the transit cluster $S$ is replaced by vertex $s$ and graph $\sG^{\prime\prime}$ where transit clusters $S$ and $T$ are replaced by vertices $s$ and $t$, respectively. It is easy to check that set $\{s, t\}$ is a transit cluster in  $\sG^{\prime\prime}$ under the assumption that $\Pa^*(\rec(S)) = \Pa^*(\rec(T))$ and $\Ch^*(\emi(S)) = \Ch^*(\emi(T))$. By applying Theorem~\ref{th:modularity} to $\sG^{\prime\prime}$ we conclude that $\{s\} \cup T$ is a transit cluster in $\sG^\prime$. By applying Theorem~\ref{th:modularity} then to $\sG^\prime$, we conclude that $S \cup T$ is a transit cluster in $\sG$.
\end{proof}
\peripheralinduced*
\begin{proof}
Let $\sG^+$ be the peripheral extension graph of $T^{+}$. A new receiver is created by operations~\ref{peri:addrec}, \ref{peri:addrecemi} and \ref{peri:addrec_noemi}. A new emitter is created by operations~\ref{peri:addemi}, \ref{peri:addrecemi} and \ref{peri:addemi_norec}.  A new emitter can be created also by operation~\ref{peri:dividenode}. Operations~\ref{peri:addrec}, \ref{peri:addemi}, \ref{peri:addrecemi}, \ref{peri:addrec_noemi} and \ref{peri:addemi_norec} explicitly copy parents and children so that the conditions~\ref{itm:parents} and \ref{itm:children} of transit cluster hold for $T^+$. If $t_i$ is an emitter in operation~\ref{peri:dividenode}, it will not be an emitter in $\sG^+$. Copying the children to new vertex $t_{k+1}$ ensures that condition~\ref{itm:children} is fulfilled for $t_{k+1}$. 

Before an operation, condition~\ref{itm:no_ronsy_allowed} hold for all existing vertices in $T$. Condition~\ref{itm:no_ronsy_allowed} holds also in $\sG^+$ for all existing vertices because the operations do not change the existence of the current paths. Condition~\ref{itm:no_ronsy_allowed} holds for the new vertex $t_{k+1}$ added by operations~\ref{peri:dividenode}, \ref{peri:addparent} and \ref{peri:addchild} because condition~\ref{itm:no_ronsy_allowed} holds for $t_i$, and $t_{k+1}$ is connected to $t_i$. Condition~\ref{itm:no_ronsy_allowed} directly holds for a receiver or an emitter added by operation~\ref{peri:addrec}, \ref{peri:addemi}, \ref{peri:addrecemi}, \ref{peri:addrec_noemi} or \ref{peri:addemi_norec}.

We also conclude that if there exist a path in $\sG$ from $r \in \rec[\sG](T)$ to $e \in \emi[\sG](T)$, there also exist a path from $r$ to $e$ in $\sG^+$ because operation~\ref{peri:addedge} does not remove edges or vertices, operations~\ref{peri:addmediator} and \ref{peri:dividenode} preserve directed paths, operations~\ref{peri:addparent} and \ref{peri:addchild} do not affect receivers and emitters, operations~\ref{peri:addrec}, \ref{peri:addemi} and \ref{peri:addrecemi} explicitly add an edge to create the required directed path, and operations~\ref{peri:addrec_noemi} and \ref{peri:addemi_norec} do not apply to cases where $T$ has both receivers and emitters. It follows that the conditions~\ref{itm:forallr} and \ref{itm:foralle} of transit cluster hold for $T^+$. Graph $\sG^\prime$ is the induced graph of $\sG^+$ because $\Pa[\sG^*](\rec[\sG^*](T^*)) \setminus T^* = \Pa[\sG](\rec[\sG](T)) \setminus T$ and $\Ch[\sG^+](\emi[\sG^+](T^+)) \setminus T^+ = \Ch[\sG](\emi[\sG](T)) \setminus T$.
\end{proof}
\peripheralconstruction*
\begin{proof}
At the beginning, all vertices in $T$ are unmarked, i.e., they are not yet included in the graph to be constructed.  Assume first $\rec(T) \neq \emptyset$ and $\emi(T) \neq \emptyset$. Apply operation~\ref{peri:dividenode} to $t$ to create set $A_0$ that has exactly one receiver $r$ and one emitter $e$. Choose $r$ to be an arbitrary member of $\rec(T)$ to $r$ and choose $e$ to be an arbitrary member of $\emi(T)$ that is a descendant of $r$. Apply operation~\ref{peri:addmediator} iteratively to construct a path corresponding to the path from $r$ to $e$ in $T$. Now all vertices of $T$ that belong to this path are marked.

Form a set $A_2$ of such receiver-emitter pairs in $T$ that there is a directed path from the receiver to the emitter. While there are unprocessed pairs in $A_2$, do the following operations: If both the receiver and the emitter are unmarked, apply operation~\ref{peri:addrecemi} to create the receiver-emitter pair and apply operation~\ref{peri:addmediator} to create the directed path between them. If the receiver is unmarked and the emitter is marked, apply operation~\ref{peri:addrec} to connect the receiver to a vertex that is on the receiver-emitter path and is an ancestor of all marked vertices on this path. Then apply operation~\ref{peri:addmediator} iteratively to create all vertices of the receiver-emitter path. If the receiver is marked  and the emitter is unmarked, apply operation~\ref{peri:addemi} to connect the emitter to a vertex that is on the receiver-emitter path and is a descendant of all marked vertices on this path. Then apply operation~\ref{peri:addmediator} iteratively to create all vertices of the receiver-emitter path. If both the receiver and the emitter are marked, apply first operation~\ref{peri:addedge} and then iteratively operation~\ref{peri:addmediator} to create the directed path between them.

Next process all vertices of $T$ that have not been marked yet. Apply operations~\ref{peri:addparent} and \ref{peri:addchild} to connect them to a vertex that is their child or parent. Repeat this until there are no vertices left. Finally, process all edges of $T$ and use operation~\ref{peri:addedge} to add the missing edges.

Next assume $\rec(T) = \emptyset$ and $\emi(T) \neq \emptyset$. Apply operation~\ref{peri:addrec_noemi} to add all receivers. Then process all vertices of $T$ that have not been marked yet similar way as above. Finally process all edges of $T$ and use operation~\ref{peri:addedge} to add the missing edges.  The case $\rec(T) = \emptyset$ and $\emi(T) \neq \emptyset$ proceed analogously.
\end{proof}
\clusterdecomposition*
\begin{proof}
Because $T$ is not a transit component, there exists a set $R \subset T$ such that $G[R]$ is connected and such that $R$ is not connected to $S = T \setminus R$ in $G[T]$. Note that such a set necessarily exists, because at least one vertex $t$ in $T$ is not connected to $T \setminus \{t\}$ in $G[T]$ due to $T$ not being a transit component. Next, we show that $S$ and $R$ are transit clusters. If $T$ has receivers, then the receivers of  $S$ and $R$ must have the same parents as the receivers of $T$ because $S$ and $R$ are disconnected in $G[T]$, and because $T$ is a transit cluster, thus satisfying condition~\ref{itm:parents}. Analogously, if $T$ has emitters then the children of the emitters of $S$ and $R$ must be the same, satisfying condition~\ref{itm:children}. Conditions~\ref{itm:no_ronsy_allowed} through \ref{itm:foralle} are satisfied for $S$ and $R$ because any path from an emitter to a receiver in $T$ exists either entirely in $S$ or $R$ because $S$ or $R$ are disconnected in $G[T]$. Thus $S$ and $R$ are disjoint transit clusters such that $T = S \cup R$ and $R$ is a transit component because it is connected in $G[R]$.
\end{proof}
\compterminates*
\begin{proof}
The sets $\sV_{\text{Ch}}$ and $\sV_{\text{Pa}}$ are finite, and for any set $A$ constructed on line~\ref{line:defineA}, there is only a finite number of possible components $\sC(\sG[A])$ in the innermost for-loop. Thus, there is finite number of iterations in total across all for-loops, and all other operations are well-defined and nonrecursive.
\end{proof}
\compfindersound*
\begin{proof}
We show that if $\compfinder(\sG, R)$ reaches line~\ref{line:add_component}, then the set being added to $\sA$ is a transit component in $\sG$. Because $\sT_{\sG|R} \subseteq \sT_{\sG}$ for all $R \subset V$, we can assume that $R = V$.  Let $V_i, V_j$ be a pair defined on lines~\ref{line:loop_receiver_candidates} and \ref{line:loop_emitter_candidates} such that line~\ref{line:add_component} is triggered in the same iteration. Let $Z$ and $W$ {}be defined as dictated by lines~\ref{line:restrict_an} and \ref{line:restrict_de}, respectively. The condition on line~\ref{line:has_ch_or_pa} must not have been fulfilled, which means that if $Z$ is non-empty, all of its members have parents, and if $W$ is non-empty, all of its members have children. Because the condition on line~\ref{line:re_or_em_nonempty} is fulfilled, we know that at least one of the sets $Z$ and $W$ is non-empty.

We summarize the construction of $A$ on line~\ref{line:defineA}. The set contains all vertices connected to $Z \cup W$ when incoming edges of $Z$ and outgoing edges of $W$ have been removed in $\sG$. The intuition is to construct a set $A$ such that $Z$ would be equal to $\rec(A)$ and $W$ would be equal to $emi(A)$. The construction together with lines~\ref{line:restrict_an} and \ref{line:restrict_de} ensures that there will be a path from any member of $Z$ to some member of $W$ and vice versa, which is required to satisfy conditions~\ref{itm:forallr} and \ref{itm:foralle} of Definition~\ref{def:transitcluster}. The line~\ref{line:defineA} directly enforces condition~\ref{itm:no_ronsy_allowed}. However, this construction alone does not guarantee that $Z$ and $W$ will be the set of receivers and emitters of $A$, respectively, because $Z$ might not have the same parents outside of $A$, or $W$ might not have the same children outside of $A$. It might also be the case that $A$ is not connected in $\sG[A]$.

Next, we break $A$ into its components in $\sG[A]$ and iterate over them on line~\ref{line:component_iteration}. Conditions~\ref{itm:no_ronsy_allowed}, \ref{itm:forallr} and \ref{itm:foralle} remain valid for each component. Because line~\ref{line:add_component} is reached, there must be at least one component $A_k$ for which line~\ref{line:re_em_verify} evaluates to true. This means that for such a set $A_k$, the sets $Z_k$ and $W_k$ have the same set of parents and children outside of $A_k$ as any of their members, respectively. This means that the remaining conditions~\ref{itm:parents} and \ref{itm:children} of Definition~\ref{def:transitcluster} are satisfied by $A_k$, making $A_k$ a transit component.
\end{proof}
\compfindercomplete*
\begin{proof}
We show that if $T \subset V$ is a transit component in $\sG$, then it will be a member of the set $\sA$ returned by \compfinder{}. Because $\sT_{\sG|R} \subseteq \sT_{\sG}$ for all $R \subset V$, we can assume that $R = V$. Definition~\ref{def:transitcluster} implies that there exists a pair of vertices $v_i,v_j \in V$ such that $\rec(T) \subseteq \Ch^*(v_i)$ and $\emi(T) \subseteq \Pa^*(v_j)$ by conditions~\ref{itm:parents} and \ref{itm:children}. Denote $V_i = \Ch^*(v_i)$ and $V_j = \Pa^*(v_j)$ with respect to the definitions on lines~\ref{line:loop_receiver_candidates} and \ref{line:loop_emitter_candidates}, respectively. Conditions~\ref{itm:forallr} and \ref{itm:foralle} further imply that $\rec(T) \subseteq \An(\emi(T))$ and $\emi(T) \subseteq \De(\rec(T))$. Therefore $\rec(T) \subseteq V_i \cap \An(V_j)$ and $\emi(T) \subseteq V_j \cap \De(V_i)$. At this point, we make an important choice; if there exist multiple $v_i,v_j$ pairs that satisfy these conditions, we choose one that minimizes the corresponding intersections $V_i \cap \An(V_j)$ and $V_j \cap \De(V_i)$. More precisely, we assume that for our choice $v_i,v_j$ there does not exist $v_i^\prime, v_j^\prime$ such that $V_i^\prime \cap \An(V_j^\prime) \subset V_i \cap \An(V_j)$ and $V_j^\prime \cap \De(V_i^\prime) \subset V_j \cap \De(V_i)$. We call this the minimal representative choice in the context of this proof and illustrate the meaning of this choice with an example.
%

It is easy to verify that $B = \{r_1,e_1\}$ is a transit cluster in the graph of Figure~\ref{fig:minimal_representative}. Suppose that we had chosen $v_i = x$ and $v_j = y$ resulting in $V_i = \Ch^*(x) = \{r_1,r_2,r_3\}$ and $V_j = \Pa^*(y) = \{e_1,e_2\}$. This is a valid choice because $\rec(B) \subseteq V_i,\, \emi(B) \subseteq V_j$, and for the intersections we have that $V_i \cap \An(V_j) = V_i,\, V_j \cap \De(V_i) = V_j$. However, $x, y$ is not the pair that minimizes these intersections. Choosing instead $V_i^\prime = \Ch^*(r_2) = \{r_1\}$ and $V_j^\prime = \Pa^*(e_2) = \{e_1\}$ we have that  $\rec(B) = V_i^\prime,\, \emi(B) = V_j$ and for the intersections it holds $V_i^\prime \cap \An(V_j^\prime) = V_i^\prime$, $V_j^\prime \cap \De(V_i^\prime) = V_j^\prime$. Now $ V_i^\prime \subset V_i$ and $V_j^\prime \subset V_j$ which shows that our initial choice $x,y$ was not minimal, and $r_2,e_2$ is actually the minimizing pair.

Let $Z = V_i \cap \An(V_j)$ and $W = V_j \cap \De(V_i)$ as defined on lines~\ref{line:restrict_an} and \ref{line:restrict_de}, respectively. Because $T$ is a transit component, it must have either receivers or emitters, which by definition have parents and children outside of $T$. This means that at least one of the sets $Z$ and $W$ is non-empty and $Z$ has parents in $\sG$ or $W$ has children in $\sG$. Thus, the for-loop does not continue on line~\ref{line:has_ch_or_pa}, and we move on to line~\ref{line:re_or_em_nonempty}, which is satisfied for the same reason.

Next, we construct the set $A$ on line~\ref{line:defineA}. Importantly, we must show that $T \subseteq A$. Suppose instead that there exists $t \in T$ such that $t \not\in A$ and $t$ is not a receiver or emitter of $T$. Because $T$ is a transit component, then $t$ must be connected to $\rec(T) \cup \emi(T)$ when incoming edges of $\rec(T)$ and outgoing edges of $\emi(T)$ have been removed. Due to the construction of $A$, this leaves the only option that $t$ is connected to $\rec(T) \cup \emi(T)$ only via paths that intersect $Z \setminus \rec(T)$ or $W \setminus \emi(T)$ and is no longer connected to $\rec(T) \cup \emi(T)$ when incoming edges of $Z$ and outgoing edges of $W$ are removed. Let $Z'$ and $W'$ denote those subsets of $Z$ and $W$ that only contain vertices that intersect such paths, respectively. This means that it must also be the case that $Z' \subset T$ and $W' \subset T$ because $T$ is connected in $\sG[T]$ and thus entire connecting path is in $T$. Suppose that the path has an incoming edge to $Z' \setminus \rec(T)$ and let $t_z \in Z' \setminus \rec(T)$ be a vertex on this path. Because $t_z \in T$ also, we have a contradiction, because $t_z$ is not a receiver of $T$ but $Z'$ is not empty which means that $t_z$ must have parents that are not members of $T$. The case for the path intersecting $W' \setminus \rec(T)$ is analogous. Thus we affirm that $T \subseteq A$.

Next, we must show that $T$ is a component of $\sG[A]$. Because $T \subseteq A$, there must be a component $A_k$ of $\sG[A]$ such that $T \subseteq A_k$. Let $Z_k$ and $W_k$ be defined according to lines~\ref{line:component_re} and \ref{line:component_em}, and suppose instead that there exists $a \in A_k \setminus T$ such that $a$ is connected to $T$ in $\sG[A_k]$. Let $a'$ be a vertex in $A_k \setminus T$ on the path from $a$ to $T$ in $\sG[A_k]$ such that it is either parent or a child of $T$. As $T$ and $A_k$ are both connected, it follows from the construction of $A$ that if $a'$ is a parent of $T$, then $a' \in Z_k \setminus \rec(T)$ or if $a'$ is a child of $T$, then $a' \in W_k \setminus \emi(T)$. Suppose that $a' \in Z_k \setminus \rec(T)$. Now, because $a'$ is a parent of $T$, it must a parent of all of its receivers. Furthermore, $\Ch[{\sG[A]}](a') \subset Z_k$, i.e., $a'$ necessarily has at least one fewer child than the representative $v_i$ (mainly, $a'$ itself). Then $\Ch[{\sG[A]}](a^\prime) \cap \An[{\sG[A]}](V_j) \subset V_i \cap \An[{\sG[A]}](V_j)$, which contradicts the minimal representative choice. The case for $a' \in W_k \setminus \emi(T)$ is analogous. Hence, $T$ is a component of $\sG[A]$.

We can now deduce that $Z_k \cap T = \rec[\sG](T)$ and $W_k \cap T = \emi[\sG](T)$. If there existed $z \in Z_k \cap T$ that is not a receiver of $T$, we would have a contradiction, because $T$ is a transit component, and $z_k \in \Ch[\sG]^*(v_i)$ meaning that $T$ would be connected to $\Pa^*(T)$ via a vertex that is not its receiver. The case for $W_k$ is once again analogous. The sets $Z_{\text{Pa}}$ and $W_{\text{Ch}}$ constructed on lines~\ref{line:common_pa} and \ref{line:common_ch} are simply the common parents of $\rec[\sG](T)$ and common children $\emi[\sG](T)$, and because $T$ is a transit component, the check of line~\ref{line:re_em_verify} evaluates to true as all receivers have the same parents and all emitters have the same children in $\sG$. Finally, \compfinder{} adds $T$ to the set $\sA$ on line~\ref{line:add_component}.
\end{proof}
\comppoly*
\begin{proof}
Let $n = |V|$ and $m = |E|$. Any restrictions on the vertices that are allowed to be members of transit clusters will only lead to a decrease in runtime, so we assume that $R = V$. Because the algorithm repeatedly accesses sets of parents, children, ancestors, and descendants of the vertex sets in the input graph $\sG$, we assume that these are derived as a preprocessing step, which evaluates each vertex and edge once in the worst case, thus taking $O(n + m)$ time to construct the sets (for example, via a depth-first search). Thus any future access to these sets can be carried out in constant time. 

The maximum number of unique parent and child sets in the collections $\sV_{\text{Pa}}$ and $\sV_{\text{Pa}}$ occurs when the parents and children of each vertex are unique. Thus there are at most $n$ iterations in both of the two outermost for-loops on lines~\ref{line:loop_receiver_candidates} and \ref{line:loop_emitter_candidates}, leading to $n^2$ iterations in total. Each of the constructions and verifications on lines~\ref{line:restrict_an}--\ref{line:re_or_em_nonempty} can be evaluated in $O(n)$ time with help of the preprocessing step.

The construction of the candidate set $A$ and its components $\sC(\sG[A]))$ on lines~\ref{line:defineA} and \ref{line:component_iteration} takes $O(n + m)$ time in the worst case, when the entire graph has to be traversed (again, for example by a single depth-first search for both tasks simultaneously, when the search reaches a vertex $v$ through an incoming edge to $Z$ or an outgoing edge from $W$, it simply immediately returns to the previous vertex without discovering $v$.). There are at most $n$ components of $\sG[A]$, which leads to at most $n$ iterations in the innermost for-loop on line~\ref{line:component_iteration}. During lines~\ref{line:component_re}--\ref{line:add_component}, each operation can be carried out in $O(n)$ time, once again taking advantage of the preprocessing step.

Combining all of the previous observations gives us
\[
O\left((n + m) + n^2(n + (n + m) + n^2) \right) = O(n^4 + n^3m) = O(|V|^4 + |V|^3|E|).
\]
\end{proof}
\countingcomponents*
\begin{proof}
From the assumptions it follows that for any transit component $S$ of $\sG$, we have that either $T \subset S$ or $T \cap S = \emptyset$. Thus, by Theorem~\ref{th:stilltransitcluster}, there must be an equal number of transit components that do not contain $T$ in $\sG$ and those that do not contain $t$ in $\sG'$. Similarly by Theorem~\ref{th:modularity} there must be an equal number of transit components that contain $T$ for $\sG$ and those contain $t$ for $\sG'$. In this case we apply the theorem to $S' = S \setminus T$ and $T$ to make the previous observation. This means that the only difference in the number of transit components of $\sG$ and $\sG'$ is that in $\sG$, the set $T$ induces $|\sT_{\sG[T]}|$ transit components, whereas in $\sG'$ the set $\{t\}$ induces a single transit component. However, the difference is offset by one because $T$ is not a transit component in $\sG[T]$.
\end{proof}
\componentamount*
\begin{proof}
Because \compfinder{} is sound and complete, we take advantage of the algorithm, and consider pairs of sets $V_i$ and $V_j$ defined on lines~\ref{line:loop_emitter_candidates} and \ref{line:loop_receiver_candidates}, and the corresponding sets $Z$ and $W$. We consider cases: 1) each pair $(V_i, V_j)$ produces at most one transit component, 2) at least one pair $(V_i, V_j)$ produces more than one transit component, and in addition 3) show that the equality $|\sT_{\sG}| = |V|(|V| + 1)/2 - 1$ is attainable.

\begin{enumerate}
\item In the first case, the maximum number of pairs $(V_i, V_j)$ such that $V_i \cap \An(V_j) \neq \emptyset$ and $V_j \cap \De(V_i) \neq \emptyset$ is clearly $\frac{|V|(|V| + 1)}2 - 1$, because $\sG$ is a DAG. Assuming that each pair where at least one of the aforementioned intersections is nonempty produces a distinct transit component, the claim follows.

\item In the second case, we must consider the repercussions of multiple transit components being induced by the same pair $(V_i, V_j)$. Also, we will associate each transit component $T$ with a unique representative pair $(V_i, V_j)$ as follows: if it occurs that a transit component $T$ is induced by two distinct pairs $(V_i,V_j)$ and $(V'_i, V'_j)$, we choose the pair that induces the smallest number of transit components. In the case that there are still multiple such pairs, the choice is arbitrary. Next, we must consider two separate scenarios: a) the sets $Z$ and $W$ corresponding to the pair $(V_i, V_j)$ defined on lines~\ref{line:restrict_an} and \ref{line:restrict_de} are both nonempty, and b) one of the sets $Z$ and $W$ is empty.
\begin{enumerate}[a)]
\item Suppose that the pair $(Z,W)$ induces $k$ transit components, $A_1, \ldots, A_k$, defined according to line~\ref{line:loop_components} which are the components of $A$, defined according to line~\ref{line:defineA}. Let $n = |V|$, $n_i = |A_i|$ for each $i = 1,\ldots,k$, and $n_0 = |V| - |A|$. Suppose that for some $A_j$, there exists a transit component $T$ of $\sG$ such that $A_j \cap T \neq \emptyset$ and $A_j \setminus T \neq \emptyset$. Because $T$ is a transit component and thus connected, it must contain at least one child of an emitter of $A_j$ or a parent of a receiver of $A_j$. Without loss of generality, assume that a child $c$ of an emitter of $A_j$ is a member of $T$. However, this also makes $c$ a child of an emitter $e$ of at least one other transit component $A_i$, $i \neq j$. It cannot be the case that there would exist a child of an emitter of $A_j$ such that it is not a child of an emitter of $A_i \neq A_j$ because we have assumed that the current pair $(V_i,V_j)$ produces the smallest number of transit components. Now, if $e \not\in T$, we have a contradiction, because $A_j \setminus T \neq \emptyset$, there must a path from $A_j$ to a receiver of $T$, but there cannot be a path from $e$ to the same receiver, because $A_i$ and $A_j$ are not connected in $\sG[A]$. If $e \in T$, it follows that $A_i \subset T$ using the same argument with nodes along any path from a receiver of $A_i$ to $e$ that always exists due to the definition of a transit cluster. Now, because $Z$ is not empty, there is a parent $p$ of a receiver of $A_j$ that is also a parent of a receiver of $A_i$. If $p \not\in T$ we have a contradiction, because there is a receiver of $T$ in $A_i$ which has different parents than a receiver of $T$ in $A_j$. If $p \in T$, this makes $p$ an emitter of $T$ and we have again a contradiction, because there would be a cycle in the induced graph of $T$, as there is a directed path from an emitter to a receiver in $T$. We conclude that no such transit component $T$ can exist. The case where $T$ contains instead a parent of a receiver of $A_j$ is analogous.

We can now apply Lemma~\ref{lem:counting_components} repeatedly to each set $A_1, \ldots, A_k$. Let $\sG'$ denote the graph obtained after clustering each $A_i$, $k = 1, \ldots, k$. We have that
\[
  |\sT_{\sG}| = |\sT_{\sG'}| + \sum_{i = 1}^k |\sT_{\sG[A_i]}|.
\]
We now apply induction to the original claim in terms of the number of vertices. The base case clearly holds, and the induction assumption is that for all DAGs with $m < n$ vertices, the number of transit components is less or equal to $\frac{m(m+1)}2 - 1$. For the induction step, applying the induction assumption to the above equation yields the following inequality:
\begin{align*}
  |\sT_{\sG}| &\leq \frac{(n_0+k)(n_0+k+1)}2 - 1 + \sum_{i=1}^k \left(\frac{n_i(n_i + 1)}2 - 1\right) \\
              &= \frac{(n_0+k)(n_0+k+1)}2 - k - 1 + \sum_{i=1}^k \frac{n_i(n_i + 1)}2.
\end{align*}
Note that in the case that $n_i = 1$ for all $i = 1,\ldots,n$ we have $\sum_{i=1}^k n_i = k$, $n_0 + k =n$ and we obtain:
\begin{align*}
  \frac{(n_0+k)(n_0+k+1)}2 - k - 1 + \sum_{i=1}^k \frac{n_i(n_i + 1)}2 &= \frac{n(n+1)}2 - k - 1 + \sum_{i=1}^k \frac{1\cdot2}2 \\ &= \frac{n(n+1)}2 -k -1 + k \\ &= \frac{n(n+1)}2 - 1.
\end{align*}
In the case that at least one $n_i > 1$ we obtain instead (by multiplying both sides by 2 for convenience):
\begin{align*}
  2|\sT_{\sG}| &\leq (n_0+k)(n_0+k+1) - 2k -2 + \sum_{i=1}^k n_i(n_i + 1)\\
               &= n_0^2 + 2kn_0 + k^2 + n_0 - k -2 + \sum_{i=1}^k n_i(n_i + 1) \\
               &= n_0^2 + 2kn_0 + k^2 + n_0 - k -2 + \sum_{i=1}^k n_i^2 + \sum_{i=1}^k n_i \\
               &= \sum_{i=0}^k n_i^2 + 2kn_0 + k^2 - k -2 + \sum_{i=0}^k n_i \\
               &< \sum_{i=0}^k n_i^2 +2n_0\sum_{i=1}^k n_i + k(k-1) -2 + \sum_{i=0}^k n_i \\
               &< \sum_{i=0}^k n_i^2 +2n_0\sum_{i=1}^k n_i + \sum_{i \neq j, 0 < i,j \leq k} n_i n_j + \sum_{i=0}^k n_i - 2\\
               &= \sum_{i=0}^k n_i^2 + \sum_{i \neq j, 0 \leq i,j \leq k} n_i n_j + \sum_{i=0}^k n_i - 2\\
               &= \left( \sum_{i=0}^k n_i \right)\left( \sum_{i=0}^k n_i + 1 \right) -2\\
               &= n(n+1) -2.
\end{align*}\label{eq:numberofcomps}

The first inequality in \eqref{eq:numberofcomps} is due to $k < \sum_{i=1}^k n_i$ and the last inequality follows from the fact that there are $k(k-1)$ pairs $i,j$ such that $i \neq j$ and $0 < i,j \leq k$. The second and third terms after the second inequality can be combined by changing the summation indices, and the second to last line is just a further refactoring of the terms. Thus we have that $|\sT_{\sG}| \leq \frac{n(n+1)}2 - 1$ for all $1 \leq k \leq n$ and partitions $n_0,n_1, \ldots, n_k$ such that $\sum_{i=0}^k n_i = |V|$.

\item Without loss of generality, assume that $Z = \emptyset$. This means that each of the $k$ components have the same set of receivers, mainly, the empty set. Now consider the possible number of candidate pairs $(V_i,V_j)$, the amount of unique receivers has decreased by $k-1$, because we know the empty set to be the candidate receiver set of at least $k$ pairs $(V_i,V_j)$. Now, we have found $k$ transit components, but the potential number of remaining $(V_i,V_j)$ pairs has decreased by at least $k$ as well.

Assuming that all remaining pairs $(V_i,V_j)$ fall into the first or the previous case, the claim immediately follows. If not, we can repeat the above argument until this is the case, because the amount of new transit components found each time cannot exceed the amount that the number of undiscovered transit components decreases by. The case for $W = \emptyset$ is identical.
\end{enumerate}
\item Let $\sG = (V,E)$ consists of single directed path $v_1 \rightarrow \cdots \rightarrow v_n$. It is easy to see that any transit component of $\sG$ also takes the form of a directed path $v_a \rightarrow \cdots \rightarrow v_b$, $v_a,v_b \in V$. There are $n - 1$ possible ways to choose the length $k$ of such a path (the path of full length $n$ is not a transit component), and for each length $k$, there are $n-k+1$ ways to choose the starting vertex of the path. Thus we obtain the following:
\begin{align*}
  |\sT_{\sG}| &= \sum_{i = 1}^{n - 1} (n-i+1) \\
              &= -1 + \sum_{i = 1}^{n} (n-i+1) \\
              &= n^2 + n -1 - \sum_{i=1}^n i \\
              &= n(n+1) - 1 - \frac{n(n+1)}2 \\
              &= \frac{n(n+1)}2 -1. 
\end{align*}
\end{enumerate}
\end{proof}
\clustterminates*
\begin{proof}
Because there is a finite number of restricted transit components considered in the loop on line~\ref{line:loop_components} of \compfinder{}, it remains to show that \expander{} always terminates for valid inputs $T$, $\sA$, $\sB$, and $\sG$. First, we note that in all following recursive calls to \expander{}, the number of elements in the set $\sB'$ decreases, and eventually the set becomes empty. Thus any single branch of the recursion will eventually reach the returning line~\ref{line:return_from_recursion}. Further, because there are only a finite number of elements considered in the loop on line~\ref{line:loop_remaining_components} of \expander{}, there can only be a finite number of branches in the recursion. Thus \expander{} always terminates.
\end{proof}
\clustsound*
\begin{proof}
New members to the output set $\sA$ are added from the output of the subroutine \expander{}. This subroutine performs recursive unions of transit clusters and transit components while ensuring that the conditions of Corollary~\ref{th:union} hold, which guarantee that the union is a transit cluster. Thus, only transit clusters are ever added to the set $\sA$.
\end{proof}
\clustcomplete*
\begin{proof}
By Theorem~\ref{thm:cluster_is_component_union}, any transit cluster can be constructed iteratively from transit components. It is easy to see that Corollary~\ref{th:union} specifies the only valid union for a transit component and a transit cluster that are disjoint and not connected in the subgraph induced by their union. It follows that all possible distinct unions of transit components are considered in the recursive calls to \expander{} launched by \clustfinder{} on line~\ref{line:recursion_start}. Following this, if a union $S \cup T$ of transit components $S$ and $T$ is valid according to Corollary~\ref{th:union} on line~\ref{line:theorem_condition} of \expander{}, a new recursive call to \expander{} is launched on line~\ref{line:valid_union_recursion}, which will now consider all possible unions of the transit cluster $S \cup T$, and the remaining transit components of $\sT_{\sG|R}$. Eventually, all possible unions are considered, and thus by Theorem~\ref{thm:cluster_is_component_union}, all transit clusters are included in the set $\sA$ that is returned by \clustfinder{}  on line~\ref{line:return_clusters}.
\end{proof}
\clustpolydelay*
\begin{proof}
Recalling Theorem~\ref{th:component_amount}, we note that the number of transit components grows as $O(|V|^2)$.
We must consider four cases: 1) the initial preprocessing step that enables the application of Corollary~\ref{th:union} in $O(|V|)$ time, 2) finding the initial cluster, 3) finding any following cluster, and 4) termination after finding the last cluster. Let $k = \frac{|V|(|V|+1)}2 - 1$ denote the maximum potential number of transit components of $\sG$.
\begin{enumerate}
  \item To apply Corollary~\ref{th:union} in linear time, we must first determine the parents and children of each vertex, thus requiring the traversal of the entire graph, which takes $O(|V| + |E|)$ time. We assume that the input $|\sT_{\sG|R}|$ is encoded in such a way, that the sets $\rec(T)$ and $\emi(T)$ can be obtained in constant time for any transit component (i.e., the transit components ``know'' their own sets of receivers and emitters). Once sets of parents and children for each vertex are obtained, it is easy to see that we can construct any parent set $\Pa^*(\rec(T))$ or child set $\Ch^*(\emi(T))$ in $O(|V|)$ time. Test for equality between sets is also an $O(|V|)$ operation.
  \item Before the first valid union is found, only unions of transit components are considered. There are at most $k$ possible options for the set $T$ and $k - 1$ options for the set $S$. In the worst case, none of the possible pairs produces a valid union on line~\ref{line:theorem_condition}, leading to $k(k-1)/2$ operations, because only distinct unions are considered. With the preprocessing step, Corollary~\ref{th:union} can be applied in $O(|V|)$ time (assuming a dynamic programming approach that keeps track of the emitters, receivers and their parents and children for valid unions $S \cup T$).
  \item Suppose that the sets $S$ and $T$ have produced a valid transit cluster in a recursive call to \expander{}. In the worst case, this cluster was found in the first iteration of every recursive call that preceded it, it is at the greatest depth of the recursion, i.e., the set $\sB'$ would be empty in the next call, and the next cluster is not found until the recursion exists back to the very top level and reaches its final iteration. Because the number of iterations is reduced by one at each step of the recursion, the total number of remaining iterations is $(k-1) + (k-2) + ... + 1 = (k-1)k/2$. Applying Corollary~\ref{th:union} is the same as in the previous case.
  \item This scenario is almost identical to the previous. Here, the worst case is different only in the aspect that instead of finding the next cluster, the algorithm terminates, leading to the same number of remaining iterations $(k-1)k/2$.
\end{enumerate}
In summary, the initialization of case (1) can be carried out in $O(|V| + |E|)$ time. In cases (2), (3), and (4) there are at most $(k-1)k/2$ operations, each of which can be carried out in $O(|V|)$ time, which gives us
\[
  O(|V|(k-1)k/2) = O\left(|V|\left[\frac{|V|(|V|-1)}2 - 1\right]\frac{|V|(|V|-1)}2\right) = O(|V|^5).
\]
\end{proof}
\identifiability*
\begin{proof}
\begin{description}
\item[Plain transit cluster:] 
We show that neither $T$ in $\sG$ nor $t$ in $\sG'$ can be a part of a hedge for $p(x_A \cond \doo(x_B))$ and that clustering cannot affect hedges outside $T$. A hedge consists two c-components that fulfill certain conditions. Importantly, one of these c-components shares some vertices with set $B$. Condition $\Pa^*(\rec(T)) \cup \emi(T) \subseteq V$ in Definition~\ref{def:plaintransitcluster} guarantees that there is no variable in $U$ with children both in $T$ and $V \setminus T$. Consequently, there cannot be a c-component that contains members from both $T$ and $B$. Similarly in $\sG'$, there is no variable in $U'$ with children both in $\{t\}$ and $V \setminus \{t\}$. 

It remains to show that clustering does not change existing c-components. Assume that $v_i,v_j \in V \setminus T$ belong to same c-component and are therefore connected by a path specified in Definition~\ref{def:c-component}. Condition $\Pa^*(\rec(T)) \cup \emi(T) \subseteq V$ guarantees that such a path cannot contain members $T$ and cannot be affected by clustering of $T$. We conclude that clustering a plain transit cluster does not change the identifiability properties of $p(x_A \cond \doo(x_B))$.

\item[Congested transit cluster:] 
We need to show that there exists a hedge in $\sG$ if and only if there exists a hedge in $\sG'$. More precisely, we have to prove that c-components and their root sets fulfill the requirements of hedges in $\sG$ and $\sG'$

We first show $T$ is a subset of the vertex set $C$ of a c-component in $\sG$ if and only if $\{t\}$ is a subset of the vertex set $C'$ of a c-component in $\sG'$. 
If $C = T$ or $C'=\{t\}$ we may use similar reasoning as in the first part of the proof to conclude that neither $T$ in $\sG$ nor $t$ in $\sG'$ can be a part of a hedge for $p(x_A \cond \doo(x_B))$. Assume next $C \setminus T \neq \emptyset$. As $\emi(T) \subseteq V$ by Definition~\ref{def:congestedtransitcluster}, there must exist $u_i \in U \setminus T$ such $u_i$ has one child in $T$ and one child in $V \setminus T$. Consequently, in $\sG'$, there is an edge $u_i \rightarrow t$ and $t \in C'$. The vertices outside $T$ are unaffected and $C \setminus T = C' \setminus \{t\}$. Similarly, if we assume $C' \setminus \{t\} \neq \emptyset$,  there is an edge $u_i \rightarrow t$ and $t \in C'$. By Definition~\ref{def:congestedtransitcluster}, all vertices of $T$ belong to the same c-component whose vertex set must be $C$. 

We have to also show that clustering cannot change the properties that a c-component is required to have in order to be a hedge. Consider a hedge in $\sG$ and let $Q$ denote its root set (sink). Suppose that there exists $t_i \in T \cap Q$. If there exists in $\sG$ a directed path from $B$ to $t_i$ in the larger c-forest of the hedge then such a path will also exists in $\sG'$ from $B$ to $t$ because $t_i$ must be on a path from $\rec(T)$ to $\emi(T)$. Any path from $B$ to $Q \setminus T$ remain unchanged by the clustering, thus a corresponding hedge can be constructed in $G'$ with the root set $(Q \setminus T) \cup \{t\})$. 

Consider next a hedge in $\sG'$ that contains $t$.   It follows that vertex $t$ must have at least one parent in $\sG$. If $t$ is in the root set $Q$, we may construct a corresponding hedge in $\sG$ by choosing a receiver as the root set. If $t$ is not in the root set $Q$, the hedge must have a path $v_i \rightarrow t \rightarrow v_j$. This path can be replaced in $\sG$ by a path $v_i \rightarrow r \rightarrow \ldots \rightarrow e \rightarrow v_j$, where $r \in \rec(T)$ and $e \in \emi(T)$. Combining the conclusions above, it follows that set $T$ is a part of a hedge in $\sG$ if and only if vertex $t$ is a part of a hedge in $\sG'$. 

It remains to show that clustering cannot affect hedges outside $T$. Assume that $v_i,v_j \in V \setminus T$ belong to same c-component and are therefore connected by a path specified in Definition~\ref{def:c-component}. Condition $\emi(T) \subseteq V$ together with the fact that members of $\rec(T)$ must be observed by Definition~\ref{def:causalmodel} guarantees that such a path cannot contain members $T$ and cannot be affected by clustering of $T$. We conclude that clustering a congested transit cluster does not change the identifiability properties of $p(x_A \cond \doo(x_B))$.
\end{description}
\end{proof}
\extensionid*
\begin{proof}
First we note that peripheral extension is guaranteed to produce a causal diagram if all vertices $t_1,\ldots,t_n$ corresponding observed variables are added first and vertices $u_1,\ldots,u_n$ corresponding unobserved background variables are then added using only operations~\ref{peri:addedge} and \ref{peri:addparent} with the restriction that vertices $u_1,\ldots,u_m$ cannot not have parents.

The first  claim follows directly from Theorem~\ref{th:identifiability}. 

For the second claim, we show the do-calculus derivation for obtaining $g(p(x_{V'}))$ in $\sG^\prime$ is valid in $\sG$ when the relevant sets are replaced by clustering equivalent sets.

We show that the d-separation $V_1^\prime \independent V_2^\prime \cond V_3^\prime$ in $\sG'$ implies $V_1 \independent V_2 \cond V_3$ in $\sG$ where $V_1^\prime,V_2^\prime,V_3^\prime$ are the clustering equivalent sets of $V_1,V_2,V_3 \subset V$, respectively. 
Consider a path from $V_1^\prime$ to $V_2^\prime$ in $\sG'$ that is blocked on the condition of $V_3^\prime$. Now there are four options that arise from the definition of d-separation: 1) The path does not contain $t$. 2) The path contains $t$ but $t$ does not block the path. 3) The path contains a chain or a fork where the middle vertex $t$ belongs to $V_3^\prime$. 4)  The path contains a collider $t$ and  $\De_{\sG^\prime}(t) \cap V_3^\prime = \emptyset$.  

1) The same path exists also in $\sG$ and is blocked. 2) The definition of a transit cluster guarantees that the corresponding path exists in $\sG$. The path is blocked by the same vertex in both $\sG$ and $\sG'$. 3) In the case of a chain $v_1 \rightarrow t \rightarrow v_2$, vertex $v_1$ is a parent of a receiver and $v_2$ is a child of an emitter. In $\sG$, we have a corresponding directed path $v_1 \rightarrow r \rightarrow \ldots \rightarrow e \rightarrow v_2$ where $r \in T$ is a receiver and $e \in T$ is an emitter. In the case of a fork $v_1 \leftarrow t \rightarrow v_2$, both $v_1$ and $v_2$ are children of an emitter. In $\sG$, we have a corresponding fork $v_1 \leftarrow e \rightarrow v_2$ where $e \in T$ is an emitter.   In both cases, the path is blocked because $T \subseteq V_3$. 4) Collider $v_1 \rightarrow t \leftarrow v_2$ implies that $t$ is a receiver in $\sG'$. In $\sG$, we have a corresponding collider  $v_1 \rightarrow r \leftarrow v_2$ where $r \in T$ is receiver. Assumption $t \notin V_3^\prime$ implies $T \cap V_3 = \emptyset$. For the descendants it holds $\De_{\sG^\prime}(t) = De_{\sG}(r) \setminus T$. We conclude that the path is blocked in $\sG$.

The rules of do-calculus operate in graphs obtained from $\sG'$ by removing some edges. There are three possibilities when each edge removed is considered: i) the edge does not involve $t$, ii) the edge $(v,t)$ is an incoming edge of $t$, and iii) the edge $(t,v)$ is an outgoing edge of $t$. In case i) the removed edge exists also in $\sG$ and can be removed there as well. In case ii), all edges from $v$ to $\rec(T)$ are removed in $\sG$. This breaks all undirected paths that correspond a path containing edge $(v,t)$ in $\sG'$. In case iii), all edges from $\emi(T)$ to $v$ are removed in $\sG$.  This breaks all undirected paths that correspond a path containing edge $(t,v)$ in $\sG'$. Together with the d-separation properties proved above, this guarantees that the do-calculus derivation applicable in $\sG'$ is also applicable in $\sG$ and $g(p(x_V^\prime))$ is an identifying functional for $p(x_A \cond \doo(x_B))$ in $\sG$.
\end{proof}
\section{Hedges and Causal Effect Identifiability} \label{app:hedges}
We introduce c-trees, c-forests and hedges as defined by \citet{shpitser2006} using our notation. We begin by defining maximal c-components. 
\begin{definition}[Maximal c-component]
Let $\sG = (V \cup U, E)$ be a causal diagram. Then a subgraph $\sG^\prime = (V^\prime, E^\prime)$ of $\sG$ is a maximal c-component if it is a c-component and if $\sG^* \subseteq \sG^\prime$ for all c-components $\sG^* = (V^* \cup U^*,E^*)$ such that $V^\prime \cap V^* = \emptyset$.
\end{definition}
Next, we define c-trees which characterize direct effects.
\begin{definition}[c-tree] Let $\sG$ be a causal diagram such that it has only one maximal c-component, and such that every observed node has most one child. If there exists a node $y$ such that $\sG[\An(y)] = \sG$ then $\sG$ is a \emph{$y$-rooted c-tree}.
\end{definition}
C-forest is the generalization of a c-tree where the root $y$ is extended to a \emph{root set}, i.e., the set of nodes $\{v \in V \mid \De^*(v)_\sG = \emptyset\}$ for a causal diagram $\sG = (V \cup U,E)$.
\begin{definition}[c-forest] Let $\sG = (V \cup U,E)$ be a causal diagram and let $Y$ be its root set. If $\sG$ is c-component and every observed node has at most one child, then $\sG$ is a \emph{$Y$-rooted c-forest}. 
\end{definition}
A special pair of c-forests can be used to characterize general causal effects of the form $p(x_A \cond \doo(x_B))$.
\begin{definition}[Hedge] Let $\sG = (V \cup U,E)$ be a causal diagram, and let $A, B \subset V$ be disjoint subsets. If there exists two $R$-rooted c-forests $\sF = (V_\sF \cup U_\sF,E_\sF)$ and $\sF^\prime = (V_{\sF^\prime} \cup U_{\sF^\prime},E_{\sF^\prime})$ such that $V_\sF \cap B \neq \emptyset$, $V_{\sF^\prime} \cap B = \emptyset$, $\sF^\prime \subseteq \sF$ and $R \subset \An[{\sG[\overline{B}]}](A),$ then $\sF$ and $\sF^\prime$ form a \emph{hedge} for $p(x_A \cond \doo(x_B))$ in $\sG$.
\end{definition}
Hedges completely characterize the identifiability of causal effects from the joint distribution over the observed variables of the causal model.
\begin{theorem}[Hedge criterion, Corollary~3 of \citep{shpitser2006}]
$p(x_A \cond \doo(x_B))$ is identifiable from $p(x_V)$ in $\sG$ if and only if there does not exist a hedge for $p(x_{A^\prime} \cond \doo(x_{B^\prime}))$ in $\sG$, for any $A^\prime \subseteq A$ and $B^\prime \subseteq B$.
\end{theorem}
\bibliography{references}
\end{document}